\documentclass{article}

\PassOptionsToPackage{numbers,round}{natbib}
\usepackage[preprint]{neurips_2026}

\usepackage[utf8]{inputenc}
\usepackage[T1]{fontenc}
\IfFileExists{courier.sty}{\usepackage{courier}}{}
\usepackage{microtype}
\usepackage{amsmath,amssymb,amsthm,bm}
\usepackage{booktabs}
\usepackage{graphicx}
\ifdefined\pdfsuppressptexinfo\pdfsuppressptexinfo=-1\fi
\usepackage{multirow}
\usepackage{algorithm}
\usepackage{algpseudocode}
\usepackage{xcolor}
\definecolor{fail}{HTML}{B3261E}
\definecolor{rust}{HTML}{A8501E}
\usepackage{hyperref}
\usepackage{enumitem}

\hypersetup{colorlinks=true,linkcolor=blue!45!black,citecolor=blue!45!black,urlcolor=blue!45!black}
\hypersetup{%
  pdftitle={Verify Claims, Not Scores: Evidence-Based Verification of Modular Agents},
  pdfauthor={Ali Atiah Alzahrani},
  pdfsubject={},
  pdfkeywords={agent verification; reliable agent development; verifier robustness; environment-grounded evaluation; modular decision systems; oracle interventions; runtime verification; heterogeneous verification signals},
  pdfcreator={},
  pdfproducer={}}

\theoremstyle{plain}

\newtheorem{condcert}{Conditional certificate statement}
\newtheorem{proposition}{Proposition}
\newtheorem{corollary}{Corollary}
\newtheorem{lemma}{Lemma}
\theoremstyle{definition}
\newtheorem{assumption}{Assumption}
\newtheorem{definition}{Definition}
\theoremstyle{remark}
\newtheorem{remark}{Remark}
\newtheorem{theorem-restate}{Theorem}
\newtheorem*{prop-restate}{Proposition~\ref{prop:union} (restated)}

\newcommand{\R}{\mathbb{R}}
\newcommand{\E}{\mathbb{E}}
\newcommand{\Prob}{\mathbb{P}}
\newcommand{\CVaR}{\mathrm{CVaR}}
\newcommand{\VaR}{\mathrm{VaR}}
\newcommand{\hmu}{\hat{\mu}}
\newcommand{\hSig}{\hat{\Sigma}}
\newcommand{\Ltail}{\mathcal{L}_{\mathrm{tail}}}

\newcommand{\wt}{w_t}
\newcommand{\wsol}{w^{\mathrm{sol}}}
\usepackage{xspace}
\newcommand{\arc}{\textsc{ARC-Agent}\xspace}

\title{Verify Claims, Not Scores:\\Evidence-Based Verification of Modular Agents}

\author{Ali Atiah Alzahrani$^{\dagger}$}

\begin{document}
\maketitle
\begingroup
\renewcommand{\thefootnote}{$\dagger$}%
\footnotetext{Corresponding author: \texttt{aliatiah100@gmail.com}. The views
expressed are those of the author and do not reflect the views of any other
individual or entity. This material is for research purposes only and does not
constitute investment advice.}%
\endgroup

\begin{abstract}
When developers change one component of an agent, such as its controller, a learned model or its
verifier, they usually judge the change by an aggregate task score. That score cannot tell whether
improvement was attainable, which component lost value, or what the agent's own checks certify. We
introduce a \emph{claim-specific verification audit} for modular agents that plan, act, check and
refine. Instead of scoring the agent, the audit scores the \emph{evidence}: each conclusion is
recorded with the evidence behind it, one of four verdicts (supported, unsupported, unresolved or not
evaluated) and the boundary within which it holds. Three tools supply that evidence. Oracle policies
measure attainable improvement under an explicitly stated action set, so that a low value can be
traced to the evaluation rather than to the environment. Replacing one component at a time with a
perfect counterpart locates lost value, with null results read as unresolved whenever a downstream
component could mask them. A separate test asks whether the verifier's score identifies the quantity
it is read as bounding. Applied to a constrained portfolio-allocation agent in a synthetic market with
known hidden regimes, the audit shows that the value of perfect regime information depends on the
action set used to measure it, that the scenario generator discards most of the regime signal while
better local fidelity does not improve decisions, and that the runtime verifier can be bypassed with
no visible change in outcomes. The contribution is the protocol and the evidential distinctions it
enforces; the empirical findings are specific to the agent and environment studied.
\end{abstract}

\section{Introduction}\label{sec:intro}
Agent systems are usually improved by changing one component, such as the controller, a tool, a
memory, a learned model or a verifier, and then comparing an aggregate task score before and after
\citep{liu2024agentbench,rabanser2026reliability}. How much such a score can establish is widely
questioned. Benchmark scores often stand in for broader constructs than they measure
\citep{jacobs2021measurement,raji2021everything}, agent evaluation harnesses contain errors that
change conclusions \citep{elhattami2025webarenaverified,kara2025warex}, and small differences are
often indistinguishable from seed variation \citep{agarwal2021precipice}. We ask a narrower question:
given a change to one component, what evidence would \emph{verify} that it helped, \emph{locate} where
value was lost if it did not, and \emph{delimit} what each verification signal can support?

An aggregate comparison is uninformative in at least three situations, each documented in a separate
literature:
\begin{itemize}[leftmargin=1.2em,itemsep=2pt,topsep=3pt]
\item \textbf{The evaluation cannot express the improvement.} The value of information depends on the
actions available to exploit it \citep{howard1966}, so an evaluation with a restricted action set
can report little headroom where much exists.
\item \textbf{A downstream component masks an upstream one.} A model can improve its own fidelity
metric without improving the decisions it feeds \citep{elmachtoub2022spo,donti2017task}, and a
better upstream module changes nothing if a downstream module ignores its output.
\item \textbf{The verifier does not measure what it certifies.} A score may not identify the target
quantity \citep{gneiting2011}, optimising against a proxy check can satisfy it without the intended
property \citep{skalse2022reward,krakovna2020specification}, and agents that re-sample until a check
passes are selecting on the check \citep{yao2023react,wang2024codeact}.
\end{itemize}
In each case the score is consistent with several incompatible explanations. Telling them apart
requires the true hidden state, the best decision attainable given that state, and the ability to
replace a component with a perfect version of itself. We therefore work in a controlled setting that
provides all three: a synthetic market with recorded hidden regimes, in which an agent allocates a
portfolio under explicit constraints. The agent contains no language model, and our claims concern
modular decision agents of this kind, although the questions the audit asks are not specific to
finance.

\paragraph{Approach and what is new.} We propose a \emph{claim-specific verification audit}. Its
individual tools (oracle policies, component replacement, factorial attribution, conformal
calibration) are established. What is new is the object being verified: not the agent's score, but
each \emph{verification claim}, recorded as claim, evidence, verdict and boundary. Making the claim
the unit of verification has three consequences that, to our knowledge, have not been made explicit
together for agent development:
\begin{enumerate}[leftmargin=1.4em,itemsep=2pt,topsep=3pt]
\item \textbf{Headroom is a property of the harness.} Measuring oracle value at fixed information
across a lattice of action sets turns a null result into a testable statement about the evaluation,
which we call \emph{verifier expressiveness}, rather than a statement about the environment
(\S\ref{sec:lattice-main}).
\item \textbf{Null interventions are non-identifying by default.} Component replacements are read as
\emph{recovery}, not ablation: a positive result localises an observed bottleneck, whereas a null
result is recorded as \emph{unresolved} whenever a downstream module could mask it
(\S\ref{sec:framework-interv}).
\item \textbf{The verifier is itself a verification target.} The audit separates sampling error, model
misspecification and \emph{score identification}, meaning whether a check measures what it is read as
bounding, and keeps decision quality, calibration, admissibility and provenance as distinct claims
(\S\ref{sec:theory}). This exposes a failure that neither outcome metrics nor feasibility and
provenance records detect (\S\ref{sec:adversarial}).
\end{enumerate}

\paragraph{Case study.} The system under audit, \arc, is a reference agent that we built for this
study. It is the object of the audit rather than a proposed allocation method, and we make no
performance claim for it (\S\ref{sec:bench}; App.~\ref{app:setup}). The audit finds that:
\begin{itemize}[leftmargin=1.2em,itemsep=1pt,topsep=3pt]
\item perfect regime information has a small, statistically unresolved value under a two-portfolio
evaluation and a clearly positive value under regime-specific weights (\S\ref{sec:oracle});
\item the scenario generator keeps less than a third of the regime signal it receives, and improving
its local fidelity does not detectably improve decisions (\S\ref{sec:bottleneck}, \S\ref{sec:mech});
\item a joint repair of two generator outputs helps, while their interaction remains unresolved
(\S\ref{sec:masked});
\item bypassing the runtime verifier voids the decision-quality claim with no visible effect on
outcomes (\S\ref{sec:adversarial}).
\end{itemize}

\section{Related work}\label{sec:related}
\textbf{Evaluating agents versus verifying a change.} Benchmarking performance, verifying that a
\emph{particular} change helped, and diagnosing \emph{why} value was lost are different activities;
an aggregate score mainly serves the first \citep{liu2024agentbench,rabanser2026reliability}.
Construct-validity accounts ask what an operationalisation measures
\citep{jacobs2021measurement,raji2021everything}; we ask the same of a \emph{diagnostic} protocol.
Related work repairs evaluation harnesses \citep{elhattami2025webarenaverified,kara2025warex},
measures gaps between visible and held-out checks \citep{zhao2026specbench}, studies reward hacking
\citep{skalse2022reward,krakovna2020specification}, and monitors agents at runtime
\citep{koohestani2025agentguard,zheng2025agentsight}. Because seed counts and pairing decide which
conclusions survive \citep{agarwal2021precipice}, every result here states its seeds, pairing and
status (Table~\ref{tab:seeds}).

\textbf{Uncertainty a verification signal can address.} Separating uncertainty that more data can
reduce from uncertainty it cannot is standard \citep{derkiureghian2009,huellermeier2021}, and
imprecise-probability and partial-identification traditions \citep{walley1991,shafer1976,manski2003}
formalise the set-valued claims our verdicts take. Distribution-free methods
\citep{vovk2005,gibbs2021,angelopoulos2023} control sampling error and the \emph{consequences} of
misspecification; we add the prior question of whether a score \emph{identifies} the quantity it is
read as bounding (\S\ref{sec:theory}).

\textbf{Established tools, a new object.} The value of information depends on the actions that can
use it \citep{howard1966}, and decision-focused learning shows that fidelity and decision quality
come apart \citep{elmachtoub2022spo,donti2017task,amos2017optnet}. Oracles, component replacement,
factorial designs and Shapley attribution \citep{shapley1953,sundararajan2020shapleytaylor} are all
established, and our information--action lattice is elementary as mathematics.
Table~\ref{tab:position} summarises what we add: a protocol that binds each tool to a named claim,
an identification assumption, a verdict and a boundary.

\begin{table}[tbp]\centering\scriptsize
\setlength{\tabcolsep}{3pt}
\caption{\textbf{What each practice is typically designed to establish.} Entries describe the usual
focus of a family of methods, not the limits of every work within it. \checkmark: addressed by
design; ``partly'': addressed for a specific case; --: not the focus.}
\label{tab:position}
\begin{tabular}{@{}p{0.21\linewidth}p{0.22\linewidth}cccc@{}}
\toprule
Practice & Question it answers & \shortstack{Null traced\\ to harness} & \shortstack{Masking-aware\\ localisation} & \shortstack{Quality / admiss.\ /\\ provenance kept apart} & \shortstack{Verifier score\\ identification} \\
\midrule
Aggregate benchmarking & How well does the agent do? & -- & -- & -- & -- \\
Harness auditing & Is the grader correct? & partly & -- & -- & -- \\
Component ablation & How much does removal hurt? & -- & -- & -- & -- \\
Runtime verification & Does a run meet a spec? & -- & -- & partly & -- \\
Decision-focused learning & Does fidelity serve decisions? & -- & partly & -- & -- \\
\midrule
\textbf{Claim-specific audit} & Which claims does the evidence support? & \checkmark & \checkmark & \checkmark & \checkmark \\
\bottomrule
\end{tabular}
\end{table}

\section{The system and environment under audit}\label{sec:bench}\label{sec:prereg}
Before presenting the audit, we describe the concrete system it is applied to. Every component and
signal referred to later is introduced here.

\begin{figure}[tbp]
\centering
\includegraphics[width=0.95\linewidth]{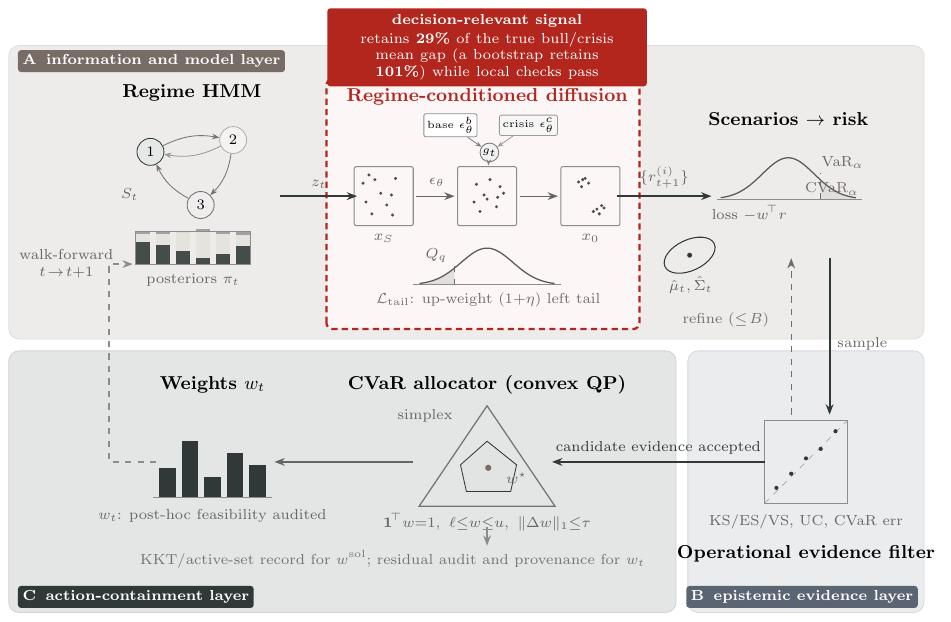}
\caption{\textbf{\arc, the agent under audit, and its three verification layers.} Arrows trace one
rebalancing decision at date $t$. \textbf{(A) Information and model layer}: a regime HMM turns past
returns into regime posteriors $\pi_t$ (walk-forward), and a regime-conditioned diffusion generator,
whose base and crisis experts are mixed by a gate $g_t$ and trained with a tail-weighted loss
$\Ltail$, draws return scenarios from which the risk quantities ($\VaR_\alpha$, $\CVaR_\alpha$) and
moments $(\hmu_t,\hSig_t)$ are computed. The red dashed box marks where the audit finds
decision-relevant signal lost: the generator retains $29\%$ of the true bull--crisis mean gap (a
bootstrap retains $101\%$) while local checks pass (\S\ref{sec:bottleneck}). \textbf{(B) Epistemic
evidence layer}: the operational evidence filter scores each scenario set with calibration checks
(KS, energy and variogram scores, unconditional coverage, CVaR error) and either accepts it or requests
refinement, at most $B$ times; it certifies no $\CVaR_\alpha$ bound (\S\ref{sec:theory}).
\textbf{(C) Action-containment layer}: the CVaR allocator solves a convex QP under the budget, box and
turnover constraints. Its KKT and active-set record is optimality evidence for the solver output
$\wsol$, whereas the committed weights $w_t$ receive only a post-hoc feasibility audit, which supports
admissibility for the reported runs.}
\label{fig:arch}
\end{figure}

\paragraph{Environment.} We use U3, a synthetic benchmark (App.~\ref{app:u3}). Ten assets, three of
them defensive, have daily returns whose distribution depends on a hidden market \emph{regime}: bull,
stagnation or crisis. The regime switches according to a Markov chain. Crises are rare (about $5\%$ of
days), last about $45$ trading days on average, and are fat-tailed, with risky assets falling
together. Because the simulator records the true regime, three things become possible that are
impossible on historical data: regime inference can be scored directly, policies that know the regime
can be computed exactly, and any component can be replaced by an exact counterpart. U3 is
verification infrastructure, not evidence about financial markets. Each run (seed) simulates a
$756$-day warm-up followed by $69$ monthly decisions, with all estimation strictly walk-forward.

\paragraph{Decision problem and metric.} At each rebalance $t$ the agent chooses weights $w_t$ that
are long-only and fully invested (the \emph{budget}), lie in a per-asset \emph{box} $[0,0.25]$, and
move at most $0.20$ in $\ell_1$ norm from $w_{t-1}$ (the \emph{turnover cap}), with $8$\,bps costs.
The metric is the Sharpe ratio, mean return divided by volatility. Because Sharpe can be bought with
deeper drawdowns, we compare at matched risk where possible, against a \emph{static frontier}: the
best Sharpe of fixed equal-weight/minimum-variance blends at the same maximum drawdown (MaxDD).

\paragraph{The agent under audit.} \arc makes each decision in five steps (Fig.~\ref{fig:arch}):
\begin{enumerate}[leftmargin=1.4em,itemsep=1pt,topsep=2pt]
\item a \emph{regime detector}, a hidden Markov model (HMM), outputs a posterior over regimes;
\item a \emph{controller}, a transparent rule plus a contextual bandit, sets configuration parameters
such as the tail weight and risk level, but never weights or constraints;
\item a \emph{scenario generator}, a regime-conditioned diffusion model, draws $N{=}512$ scenarios of
next-period returns;
\item a \emph{runtime gate}, calibrated by adaptive conformal inference on realised losses, accepts or
rejects the scenario set; a rejected set is regenerated up to $B{=}3$ times, then replaced by a
block-bootstrap fallback;
\item an \emph{allocator} solves a convex program trading expected return against variance and
conditional value-at-risk ($\CVaR_\alpha$, the mean loss in the worst $5\%$ of scenarios) under the
constraints; its output $\wsol$ is clipped and renormalised into the committed action $w_t$.
\end{enumerate}
App.~\ref{app:setup} gives the equations and the full loop (Algorithm~\ref{alg:loop}).

\paragraph{Trust boundary and signals.} Adaptive components (the controller and layer A in
Fig.~\ref{fig:arch}) may be wrong or adversarial. They can degrade decision quality but cannot enlarge the feasible set, because
the constraint parameters $(\ell,u,\tau)$ and the budget $B$ belong to a trusted validator in the action-containment layer (C). The
system emits signals bearing on four distinct properties: \emph{decision quality} (is $w_t$ a good
risk-adjusted choice?), \emph{calibration} (do scenarios match outcomes? monitored by the gate),
\emph{admissibility} (does $w_t$ satisfy the constraints? checked by a post-hoc audit) and
\emph{provenance} (can the solver output be traced to a logged optimisation? recorded by the KKT
optimality conditions).

\paragraph{Protocol.} A \emph{paired} contrast compares two arms on identical simulated paths (common
random numbers). Unless marked otherwise, intervals are $95\%$ bootstrap intervals over seeds.
Experiments with $10$--$24$ paired seeds are \emph{confirmatory} and those with $5$--$6$ seeds
\emph{exploratory} (Table~\ref{tab:seeds}). Preregistration is author-reported, with one disclosed
amendment (App.~\ref{app:trials}).

\section{A claim-specific verification audit}\label{sec:framework}\label{sec:claims}
The audit takes a proposed conclusion, such as ``the generator is the bottleneck'', and asks whether
the evaluation can \emph{express} the improvement it concerns, whether the evidence \emph{identifies}
that improvement, and where the inference \emph{stops}. For \arc we pose four questions
(evidence and permitted claims in Table~\ref{tab:constructs}, App.~\ref{app:questions}):
\begin{enumerate}[label=\textbf{Q\arabic*},leftmargin=2em,itemsep=0pt,topsep=2pt]
\item Was improvement attainable under an action set able to express it?
\item Which component replacement recovers downstream value?
\item Does improved local fidelity verify downstream improvement?
\item Do joint repairs reveal structure that one-at-a-time checks miss?
\end{enumerate}

\subsection{Verification claims and verdicts}\label{sec:ledger}
A \emph{verification claim} is a conclusion recorded as
$\text{Claim}\to\text{Evidence}\to\text{Verdict}\to\text{Boundary}$, with one of four verdicts:
\begin{itemize}[leftmargin=1.2em,itemsep=1pt,topsep=2pt]
\item \emph{Supported}: the evidence supports the claim, which may be relied on only inside its
boundary.
\item \emph{Unsupported}: the evidence conflicts with the claim, or a required identification
condition is unmet; the claim must not be used for selection or certification.
\item \emph{Unresolved}: the experiment cannot distinguish the relevant explanations; the claim needs
more evidence, not clearance. It is not a weak form of \emph{supported}.
\item \emph{Not evaluated}: no experiment was run.
\end{itemize}
Table~\ref{tab:ledger} is the resulting ledger for \arc; App.~\ref{app:limits} expands each boundary.

\subsection{Oracle headroom: can the evaluation express the improvement?}\label{sec:lattice-main}\label{sec:framework-headroom}
An \emph{oracle} is a policy with perfect knowledge of the hidden state. It obeys the system's own
constraints and costs and differs from the system only in what it knows and which actions it may
take. Indexing oracles by an information set $I$ and an action set $A$ gives a lattice of attainable
values,
\begin{equation}
J^\star(I,A)=\sup_{\pi\in\Pi(I,A)}J(\pi),
\label{eq:lattice}
\end{equation}
where $\Pi(I,A)$ is the set of policies that use only information in $I$ and act in $A$, and $J$ is
the metric. $J^\star$ is monotone in both arguments (Prop.~\ref{prop:lattice}). A headroom estimate is
therefore a property of an $(I,A,J)$ cell, not of the information alone. A low estimate in one cell
cannot show that the information has little value; if headroom rises when $A$ widens at fixed $I$,
the low estimate came from the evaluation.

For \arc we fix $I$ to the true regime path and use four action sets of increasing expressiveness:
\textbf{L0}, the best static equal-weight/minimum-variance blend, with no regime knowledge;
\textbf{L1}, switching between those two portfolios with perfect regime knowledge; \textbf{L2}, a
separate weight vector $w^\star(k)$ for each regime $k$; and \textbf{L3}, the ex-post optimal weight
path.

\subsection{Oracle-replacement interventions: where is value lost?}\label{sec:framework-interv}\label{sec:attrib}
We replace one component at a time with an oracle counterpart and re-run the system. Every arm
receives the true regime, so regime detection is never the confound. The \emph{perception}
replacement feeds the true regime to the generator; the \emph{full generator bypass} hands the
allocator scenarios drawn from clean, regime-conditional history.

The evidence is asymmetric. A replacement that \emph{recovers} value identifies an observed
bottleneck. A replacement that recovers nothing leaves open whether the component matters, because a
downstream module that ignores its input masks any upstream repair. Such a null is recorded as
\emph{unresolved}, not \emph{unsupported}.

Some repairs act only jointly. The generator reaches the allocator through three channels: the
scenario set itself (tail shape and dependence), the conditional mean $\hmu$ and the covariance
$\hSig$. A complete $2^3$ factorial switches each channel independently between the generator and
clean history. With $v(S)$ the paired Sharpe gain when the channels in $S$ are repaired, we summarise
the table with Shapley values $\phi_i$ and Shapley--Taylor interaction indices $\phi_{ij}$
\citep{shapley1953,sundararajan2020shapleytaylor}, and call $\sum_{i<j}\phi_{ij}/v(N)$ the
\emph{interaction share}. These summaries are descriptive and non-additive (App.~\ref{app:attrib}).

\subsection{What can the runtime gate certify?}\label{sec:uncertainty}\label{sec:theory}
Scenarios come from the generator's model $p_\theta$, not from the true outcome law $\Prob$, so the
error in the allocator's CVaR estimate has two parts: a \emph{sampling} term, which shrinks as $N$
grows, and a \emph{misspecification} term ($p_\theta\neq\Prob$), which does not
(App.~\ref{app:proofs}). A third question sits above both: does the verification score
\emph{identify} the quantity it is read as bounding?

The gate is calibrated on realised returns, the only direct evidence about misspecification, and acts
as an \emph{operational evidence filter}. Its score compares one realised loss with a scenario CVaR.
CVaR is not elicitable \citep{gneiting2011}, so this score does not identify $\CVaR_\alpha$
(condition C5); a joint VaR--CVaR score would \citep{fissler2016}. In addition, refine-until-pass
selects on the calibration statistic, which a union bound corrects only for bounded $B$ and only if
each attempt is valid on its own (condition C6; App.~\ref{app:proofs}). The gate is therefore not a
decision-quality certificate: that claim is \emph{unsupported}, and no empirical result in
\S\ref{sec:results} depends on the gate.

\subsection{Trusted action boundary and attack surfaces}\label{sec:threat}
Controller-settable parameters are \emph{untrusted}: they can degrade quality but not enlarge the
feasible set. Constraint parameters and $B$ are \emph{trusted} and validator-owned
(App.~\ref{app:spec}; Prop.~\ref{prop:feas}). We specify five attacks: \textbf{A1} randomises
controller outputs within their ranges; \textbf{A2} poisons memory by flipping the regime posterior;
\textbf{A3} bypasses the runtime gate; \textbf{A4} is an adversarial controller that maximises
turnover; \textbf{A5} would compromise the validator. A1--A4 were executed; A5 was not.

\section{Verification findings}\label{sec:results}
All numbers are measured and scoped to U3 and the audited architecture. Each subsection answers one
question and ends with a verdict.

\begin{figure}[tbp]
\centering
\includegraphics[width=0.88\linewidth]{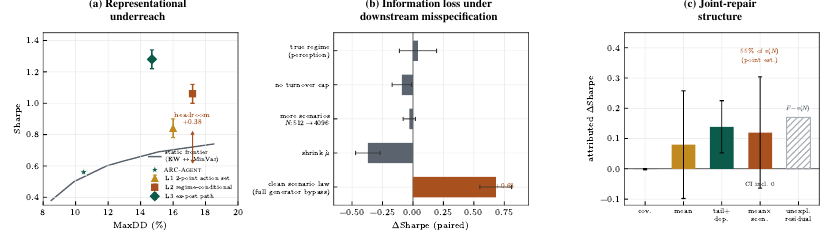}
\caption{\textbf{Three failure modes exposed by the audit} (Sharpe; seeds and intervals in
Table~\ref{tab:seeds}). \textbf{(a) Representational underreach}: regime-conditional weights (L2) sit
$+0.38$ above the static frontier at matched drawdown, but switching between two portfolios (L1) only
$+0.18$, a paired gain over L0 of $+0.07\pm0.08$. \textbf{(b) Interventional diagnosis}: the
true-regime replacement recovers nothing and is \emph{unresolved} because the generator can mask it;
removing the turnover cap, adding scenarios and shrinking the mean also fail; only the full generator
bypass recovers the gap ($+0.68$). \textbf{(c) Joint-repair structure} (exploratory; Sharpe only):
the mean-plus-scenario arm improves over baseline; the $55\%$ interaction share is a point estimate
whose six-seed interval includes zero.}
\label{fig:results}
\end{figure}

\subsection{Q1: Restricted evaluations can hide attainable headroom}\label{sec:oracle}
With perfect regime knowledge and identical constraints and costs, the four oracles reach Sharpe
$0.78$ (L0), $0.84$ (L1), $\mathbf{1.06}$ (L2) and $1.28$ (L3), against $0.56$ for \arc
(Table~\ref{tab:oracle}; $rf{=}1.6\%$, see App.~\ref{app:lattice} for $rf{=}0$).

L1 is an \emph{evaluation rule} with a restricted action set, not a weak baseline. Under it, perfect
regime information is worth a paired $+0.07\pm0.08$ Sharpe over L0, an interval that spans zero.
Widening only the action set raises the paired L2$-$L0 gain to $+0.28\pm0.08$. What limited the
measured headroom was the evaluation, not the environment. \arc does not capture this headroom (it
sits $0.04$ below the static frontier at matched drawdown), and no other scenario-based method does
either (App.~\ref{app:oracle}).

\emph{Verdict: supported}, for U3 and the stated action sets.

\subsection{Q2: Oracle replacements localise an observed generator bottleneck}\label{sec:bottleneck}
Giving the generator the true regime instead of the HMM posterior changes nothing ($16$ seeds,
paired: $+0.041\pm0.153$), although the HMM is poor (crisis recall falls from $0.54$ to $0.07$ as
tails fatten). This does \emph{not} clear perception: a generator that ignores its conditioning masks
any benefit, and the detector$\times$generator factorial was not run. Removing the turnover cap
($-0.09$), shrinking $\hmu$ ($-0.37$) and raising $N$ from $512$ to $4096$ ($-0.03$) also recover
nothing; more scenarios reduce sampling error but not misspecification. Only the full generator
bypass recovers the gap (Table~\ref{tab:intervene}).

The reason is visible in the generator's output. Given perfect regime conditioning, it reproduces
only $\mathbf{29\%}$ of the true bull--crisis difference in conditional means ($90$ of $313$\,bps),
which we call \emph{mean-gap retention}. A regime-conditional bootstrap, which resamples history from
the same regime, reproduces $\mathbf{101\%}$ and scores higher downstream ($0.85$ versus $0.56$). The
regime signal reaches the generator, yet no end-to-end score reports its loss.

\emph{Verdict: supported} for the generator as an observed bottleneck under this architecture;
\emph{unresolved} for perception.

\subsection{Q3: Local fidelity does not verify downstream improvement}\label{sec:mech}
Sweeps over diffusion steps and crisis oversampling implicate two contributors to the retention loss:
truncated reverse diffusion and under-representation of crises in training. A repair targeting both
raises retention from $\mathbf{47\%}$ to $\mathbf{186\%}$ and corrects the direction of the regime
mean shift (cosine $0.59\to0.97$). Yet the paired Sharpe change is $\mathbf{+0.075\pm0.174}$, and the
repaired generator ($0.63$) still trails the bootstrap ($0.85$). The repair overshoots the size of the
shift ($208\%$), so the mean vector's relative error grows ($0.82\to1.18$; Table~\ref{tab:retention}).

\emph{Verdict: supported.} Better local fidelity does not establish better decisions. The claim is
one of non-identification, not of a zero effect.

\subsection{Q4: Joint repairs reveal structure missed by one-at-a-time checks}\label{sec:masked}
Repairing the mean and the scenarios together improves on the all-generator baseline by $+0.221$
(six-seed paired $95\%$ CI $\pm0.160$), essentially all that the three channels recover together
($v(N)=+0.218$). Each channel alone does less (mean $+0.020\pm0.206$; scenarios $+0.079\pm0.066$;
covariance nothing). The Shapley--Taylor estimate assigns $55\%$ of $v(N)$ to the mean$\times$scenario
interaction (Table~\ref{tab:shapley}), but its six-seed interval ($\pm0.18$) includes zero, and the
factorial records Sharpe only (App.~\ref{app:factorial}).

\emph{Verdict: supported} for the joint repair improving on baseline; \emph{unresolved} for
superadditivity and the size of the interaction.

\subsection{Summary: different signals support different claims}\label{sec:hyp}
Table~\ref{tab:ledger} collects the verdicts. Evidence for one property does not transfer to another:
feasibility audits support admissibility for the reported runs only; the KKT record supports
provenance for $\wsol$, not $w_t$; and because the gate's score does not identify $\CVaR_\alpha$, the
decision-quality claim is \emph{unsupported}.

\begin{table}[tbp]\centering\scriptsize
\setlength{\tabcolsep}{3.5pt}
\caption{\textbf{Claim ledger for \arc} (boundaries expanded in App.~\ref{app:limits}).
\emph{Unresolved} is not a weaker form of \emph{supported}; it marks a claim the evidence cannot
settle either way.}
\label{tab:ledger}
\begin{tabular}{@{}p{0.30\linewidth}p{0.20\linewidth}p{0.24\linewidth}p{0.22\linewidth}@{}}
\toprule
Verification claim & Evidence & Verdict & Boundary \\
\midrule
Decision-relevant information is present and expressible & L2 information--action hierarchy & supported & U3, specified action set \\
\arc\ captures it & matched-risk frontier & \textcolor{fail}{unsupported} & current architecture \\
Generator is an observed bottleneck & full generator bypass & supported & current trained generator \\
Improving perception would not improve downstream value & true-regime replacement & \emph{unresolved} & masked by generator \\
Joint mean$+$scenario repair improves on the all-generator baseline & $2^3$ coalition table & supported & U3; 6 seeds, exploratory; Sharpe only \\
Mean$\times$scenario superadditivity & Shapley--Taylor interval & \emph{unresolved} & interval includes zero at 6 seeds \\
Executed actions are admissible & post-hoc residual checks & supported & reported runs only; not a universal guarantee \\
Solver outputs are traceable & KKT and active-set records & supported & applies to $\wsol$; auditability unmeasured \\
Implemented gate certifies decision quality (CVaR) & implemented one-step score & \textcolor{fail}{unsupported} & identification condition (C5) unmet \\
Trusted-path compromise contained & A5 & \emph{not evaluated} & attack unexecuted \\
\bottomrule
\end{tabular}
\end{table}

\section{Red-teaming the verifier and its trusted boundary}\label{sec:adversarial}
\begin{table}[tbp]\centering\small
\caption{\textbf{Attack surfaces and what survives each executed attack.} Decision quality,
admissibility and provenance are reported separately because they are different claims.
Exploratory: $5$ seeds, unpaired; levels in Table~\ref{tab:adv}. The feasibility audit is a post-hoc
check, not a construction guarantee; solver provenance is a record for $\wsol$, not KKT-optimality of
$w_t$. \textbf{A5 was not executed}, so no claim covers compromise of the trusted path.}
\label{tab:redteam}
\setlength{\tabcolsep}{3pt}
\footnotesize
\begin{tabular}{llccc}
\toprule
Attack & Claim surface & Effect on decision-quality evidence & Feas.\ audit & Solver prov. \\
\midrule
A3 gate bypass & quality-verification layer & \textcolor{fail}{\textbf{gate bypassed; metrics benign}} & \textbf{passed} & \textbf{retained} \\
A1 corrupted controller & untrusted configuration & degraded, within noise & \textbf{passed} & \textbf{retained} \\
A2 poisoned memory & untrusted memory & degraded, within noise & \textbf{passed} & \textbf{retained} \\
A4 adversarial controller & untrusted configuration & moved \emph{along} frontier & \textbf{passed} & \textbf{retained} \\
\midrule
\multicolumn{5}{l}{\emph{A5 validator compromise: specified but not executed; no robustness claim.}} \\
\bottomrule
\end{tabular}
\end{table}

If decision quality cannot be certified, what can the harness still be trusted to say? Attacks A1, A2
and A4 compromise adaptive components but leave the trusted boundary intact. Under all three, every
committed action passed the feasibility audit and kept its provenance record
(Table~\ref{tab:redteam}). This holds only for attacks that do not touch the validator.

\textbf{A3 attacks the verifier itself, and nothing notices.} Bypassing the gate removes the evidential
basis for the decision-quality claim, yet feasibility and provenance stay intact and the outcome stays
benign: under A3 the maximum drawdown, $9.9\%$, is the \emph{lowest} in the table. No outcome metric
we report would have flagged the bypass. A verifier can lose its evidential validity without any
immediately poor outcome, and feasibility and provenance cannot substitute for it, because they
never depended on the gate.

The preregistered containment hypothesis H3 was \emph{not supported}: under A4, \arc's apparent
advantage over a rule-based control ($0.83$ versus $0.72$ Sharpe) comes at a different risk level
($16.7\%$ versus $19.4\%$ MaxDD), and at matched risk the attack moved the system \emph{along} the
frontier (App.~\ref{app:redteam}).

\section{Lessons, limitations and conclusion}\label{sec:lessons}
\textbf{Lessons for reliable agent development} (in full in App.~\ref{app:principles}):
\begin{enumerate}[leftmargin=1.4em,itemsep=1pt,topsep=2pt]
\item Check that the evaluation can express the target improvement before reading anything into a
null.
\item An unchanged score fits both a useless and a masked component, so verify updates by
intervention.
\item Read interventions asymmetrically: a positive replacement localises a bottleneck; a null one
does not show irrelevance.
\item Decision quality, calibration, admissibility and provenance are different claims, and none
certifies another.
\item Red-team the verifier as well as the agent: it can be bypassed while outcomes stay benign.
\end{enumerate}

\textbf{Limitations.} Every statement is scoped to U3, one architecture and one generator family
(App.~\ref{app:limits}). Perception is unresolved under masking; the interaction results rest on six
seeds and Sharpe only; conditions C5 and C6 behind any certificate reading of the gate are unproved,
though no empirical result depends on the gate; A5 was not executed; and no language-model
controller, model judge or prompt-injection attack was evaluated.

\textbf{Conclusion.} The unit of agent verification should be the claim, not the score. Treating it
so turns a flat score into three separate questions: could the change have helped under the
evaluation used, where was value lost, and what does each check certify. In the case studied, this
exposed a restricted evaluation, an observed generator bottleneck, a gap between local fidelity and
decision value, and a verifier that could be bypassed without any visible effect on outcomes.

\begingroup
\small
\bibliographystyle{plainnat}

\endgroup

\clearpage
\appendix
\section*{Appendix}
\noindent\emph{The appendix is supplementary: it records specifications, full tables, derivations and boundary statements for the claims made in the main body.}

\section{Verification-audit principles, in full}\label{app:principles}
This case study motivates five verification-audit principles, which may transfer beyond it.
Verify the opportunity before verifying the system: our L1 oracle was a representation-restricted
evaluation protocol, not a weak baseline. Verify changes by intervention, since a null after an
upstream update establishes nothing when downstream masking is possible. Treat interactions as
first-class verification objects, because component-wise checks can obscure joint-repair
patterns; in the present factorial the joint arm improves over baseline, while the interaction
and conditional-contribution claims remain unresolved. Separate quality, admissibility and
provenance, because a feasible or auditable action may still be poor. And verify the verifier:
bypass, adaptive selection and invalid scoring rules are attack surfaces in their own right, and
our own gate fails two of them.

\section{Claim--evidence map}\label{app:questions}

Each question in \S\ref{sec:results} is answered by a specific, separately reported piece of
evidence; Table~\ref{tab:ledger} gives the resulting verdicts and their boundaries.
Table~\ref{tab:constructs} states, for each of the four questions posed in \S\ref{sec:framework}, the
evidence used and the claim that evidence permits.

\begin{table}[h]\centering\scriptsize
\setlength{\tabcolsep}{3pt}
\caption{\textbf{Verification questions posed of \arc, the evidence used, and the claim that
evidence permits.} Each row concerns a \emph{different} construct. Sharpe is an operational proxy for
risk-adjusted decision value, not a general measure of agent quality. Verdicts are collected in
Table~\ref{tab:ledger}.}
\label{tab:constructs}
\begin{tabular}{@{}p{0.04\linewidth}p{0.245\linewidth}p{0.335\linewidth}p{0.30\linewidth}@{}}
\toprule
 & Verification question & Evidence & Permitted claim \\
\midrule
Q1 & Was improvement attainable under an action set able to express it? &
L0--L3 oracle Sharpe at fixed information, widening the action set (\S\ref{sec:oracle},
Table~\ref{tab:oracle}); 24 seeds, paired, confirmatory &
Measured headroom depends on the action representation, not on the information alone \\
Q2 & Which component replacement recovers downstream value? &
Stage-wise oracle replacement with the true regime given to every arm (\S\ref{sec:bottleneck},
Table~\ref{tab:intervene}); 10--16 seeds, paired, confirmatory &
A positive replacement identifies an observed bottleneck; a null one does not identify
irrelevance \\
Q3 & Does improved local fidelity verify downstream improvement? &
Retention, directional cosine and vector error against paired Sharpe (\S\ref{sec:mech},
Table~\ref{tab:retention}); 6 seeds, exploratory &
Improved local scores do not establish downstream improvement \\
Q4 & Do joint repairs reveal structure that one-at-a-time checks miss? &
Complete $2^3$ coalition table with Shapley and Shapley--Taylor summaries (\S\ref{sec:masked},
Tables~\ref{tab:factorial}, \ref{tab:shapley}); 6 seeds, exploratory &
The joint arm improves over the all-generator baseline; its conditional contribution and the
interaction magnitude remain unresolved \\
\bottomrule
\end{tabular}
\end{table}

\begin{center}\footnotesize
\begin{tabular}{ll}
\toprule
Evaluation question & Evidence \\
\midrule
Is meaningful improvement attainable? & L0--L3 oracle hierarchy (\S\ref{sec:oracle}) \\
Did the full agent capture that improvement? & static risk frontier, paired contrasts (\S\ref{sec:oracle}) \\
Which stage forfeited the improvement? & stage-wise oracle replacements (\S\ref{sec:bottleneck}) \\
Can a module be masked by another? & complete $2^3$ channel coalition table (\S\ref{sec:masked}) \\
Does the runtime boundary preserve admissibility? & per-action feasibility audits (\S\ref{sec:adversarial}) \\
Is the action traceable? & solver records and active-set logs (App.~\ref{app:duals}) \\
Can the verifier itself fail silently? & A3 gate-bypass attack (\S\ref{sec:adversarial}) \\
Which claims remain unsupported? & unmet score identification (\S\ref{sec:theory}); unexecuted A5 \\
\bottomrule
\end{tabular}
\end{center}

\section{Oracle-replacement intervention outcomes}\label{app:intervene}
Table~\ref{tab:intervene} reports the stage-wise interventions summarised in
\S\ref{sec:bottleneck}. Every row is given the true latent regime, so detection is never
the confound, and rows are paired on common random numbers. ``\% of gap recovered'' is
measured against the L2 reference policy, not against an upper bound: L2 optimises
mean--variance per regime, so a CVaR allocator handed the clean scenario law can exceed it, which
is why I1 exceeds $100\%$. The remaining interventions of Fig.~\ref{fig:results}(b)
(scenario count, shrinkage) are reported inline in \S\ref{sec:bottleneck}.

\paragraph{The allocator is not the bottleneck (detail for \S\ref{sec:bottleneck}).} We
conjectured that the CVaR term, roughly $33\times$ the return term at a daily horizon, dominates
the objective and leaves the solution insensitive to the conditional mean. We tested it, and it
is false: swapping bull- for crisis-regime scenarios
moves the QP's solution by $118\%$ of the oracle's own rotation, aligned with it at
$\cos(\Delta w_{\mathrm{QP}},\Delta w_{\mathrm{oracle}})=+0.75$. The optimiser would therefore
have used a correct conditional mean had one reached it, which is what makes the generator, not
the allocator, the observed bottleneck of \S\ref{sec:bottleneck}.

\begin{table}[tbp]\centering\small
\caption{Oracle-replacement interventions ($10$ seeds, paired, common random numbers; true
regime given to every row). Only a full
generator bypass recovers the measured gap; no other intervention recovers a detectable
share of it. The factorial of Table~\ref{tab:factorial} subsequently separates that single
bypass into the conditional-mean, covariance and scenario tail/dependence channels; the two
experiments use different baselines and seed counts and their levels are not directly
comparable (\S\ref{sec:masked}). ``\% of gap recovered'' is
$(\text{Sharpe}-\text{Sharpe}_{\mathrm{I0}})/(\text{Sharpe}_{\mathrm{L2}}-\text{Sharpe}_{\mathrm{I0}})$.
$\pm$ values are across-seed standard errors over $10$ seeds under common random numbers. \emph{L2 is a reference policy, not an upper bound}: it optimises
mean--variance per regime, so a CVaR allocator given the clean scenario law can exceed it.}
\label{tab:intervene}
\begin{tabular}{@{}p{0.40\linewidth}ccc@{}}
\toprule
Intervention & Sharpe & Turn.\% & \% of gap recovered \\
\midrule
I0\; baseline (scenario moments, cap $0.20$) & $0.465\pm0.090$ & $19.7$ & n/a \\
I1\; $+$ \textbf{clean regime-conditional scenarios and moments} \emph{(full generator bypass)} & $\mathbf{1.141\pm0.127}$ & $18.8$ & $\mathbf{+121\%}$ \\
I2\; $+$ no turnover cap & $0.377\pm0.084$ & $60.2$ & $-16\%$ \\
\midrule
\;\;\; L2 regime-weight reference policy & $1.023\pm0.111$ & $8.8$ & n/a \\
\bottomrule
\end{tabular}
\end{table}

\section{Factorial over decision functionals}\label{app:factorial}
The allocator reads the scenario law through three separable channels, so we can swap
each independently between the generator's output and clean regime-conditional history.
Table~\ref{tab:factorial} reports all eight cells of the $2^3$ design, paired
against the all-generator baseline. The two-channel rows are the informative ones:
covariance adds nothing to any other repair, so the design collapses to a
mean$\times$scenarios interaction on a nuisance third factor. Note the asymmetry in the
single-channel intervals: the scenario-only effect is $+0.079\pm0.066$, an interval that
excludes zero, so that channel does deliver a modest standalone repair; the mean-only effect
is $+0.020\pm0.206$, which is wide enough not to exclude the joint effect either.
Descriptively, the joint-arm point estimate is larger than either single-arm point estimate.
Because the paired contrast between the joint and scenario-only arms is unavailable from the
retained aggregate summaries, this comparison does not establish a nonzero conditional mean
contribution or superadditivity.

The contrast that bears most directly on the mean channel is
$v(\{\text{mean},\text{scenario}\})-v(\{\text{scenario}\})=+0.221-0.079=+0.142$: what the
clean conditional mean adds \emph{once the scenario channel is already repaired}. We report
the point estimate only. Its seed-bootstrap interval would be a paired contrast between two
repaired cells rather than against the common baseline, and the per-seed pairs those two
cells require were not retained, so the interval cannot be recovered from the available
aggregate summaries. We therefore report this contrast descriptively and do not use it to
support the interaction claim.

\begin{table}[tbp]\centering\small
\caption{\textbf{The complete $2^3$ factorial: all eight cells.} Paired $\Delta$Sharpe
against the all-generator baseline ($6$ seeds, common random numbers, $95\%$ CI over seeds).
Reading down the single-channel rows, covariance is worth nothing; reading across the pairs,
\emph{adding} covariance to any other repair changes nothing (mean $+0.020$ versus
mean$+$covariance $+0.020$; scenarios $+0.079$ versus scenarios$+$covariance $+0.079$), while
mean$+$scenarios recovers essentially the whole effect and adding covariance on top of it
does not improve matters ($+0.221$ vs.\ $+0.218$). At six seeds the design is
\textbf{underpowered to distinguish superadditivity from additivity}. Descriptively the
joint-arm estimate exceeds either single-arm estimate; because the paired joint-versus-
scenario-only contrast cannot be recovered from the retained summaries, neither
superadditivity nor a nonzero conditional mean contribution is established. The interaction
estimate and its seed-bootstrap interval are reported in Table~\ref{tab:shapley}; at six
seeds that interval includes zero. \emph{These cells report Sharpe only.} We did not record
MaxDD or turnover per cell, so the risk-axis check we insist on elsewhere
(\S\ref{sec:oracle}, \S\ref{sec:adversarial}) has not been applied to the factorial: a
channel repair that moved a cell along the risk frontier rather than off it would not be
visible here.}
\label{tab:factorial}
\begin{tabular}{lccc}
\toprule
Channels replaced with clean regime-conditional history & Sharpe & $\Delta$Sharpe & $95\%$ CI \\
\midrule
none (all generator; baseline)                  & $0.635$ & n/a & n/a \\
covariance $\hat\Sigma$ only                    & $0.634$ & $-0.001$ & $\pm0.001$ \\
conditional mean $\hat\mu$ only                 & $0.655$ & $+0.020$ & $\pm0.206$ \\
scenarios only (tail $+$ dependence)            & $0.713$ & $+0.079$ & $\pm0.066$ \\
\midrule
mean $+$ covariance                             & $0.654$ & $+0.020$ & $\pm0.212$ \\
scenarios $+$ covariance                        & $0.714$ & $+0.079$ & $\pm0.066$ \\
\textbf{mean $+$ scenarios}                     & $\mathbf{0.856}$ & $\mathbf{+0.221}$ & $\pm0.160$ \\
all three                                       & $0.853$ & $+0.218$ & $\pm0.164$ \\
\bottomrule
\end{tabular}
\end{table}

\paragraph{Reconciling Table~\ref{tab:intervene} with Table~\ref{tab:factorial}.}
These are different experiments and their levels are not directly comparable, so we state
the relationship explicitly.
I1 is a \emph{full generator bypass}: the allocator is handed the clean regime-conditional
empirical scenario law (its scenario set together with the moments computed from it) in
place of anything the diffusion model produces. The $2^3$ factorial instead decomposes that
single bypass into its three separable channels and swaps them independently. Consequently
(i) I1 is the union of all three factorial repairs plus whatever the bypass changes that
the three-channel parameterisation does not capture, and (ii) the two arms are run at
different seed counts ($10$ vs.\ $6$) against differently-constructed baselines ($0.465$ for
I0, $0.635$ for the all-generator factorial cell), which are themselves trained separately,
as elsewhere in this paper (\S\ref{sec:masked}). The gap between I1's $1.141$ and the
factorial's all-three cell at $0.853$ is therefore \emph{not} evidence that the channels
fail to compose; it is the combined effect of a wider intervention, a different baseline,
and independent training runs. No paired arm is available to decompose that difference: with
independently trained baselines and different seed counts, the remaining discrepancy is not
identifiable from the reported runs, and the boundary is recorded as item (ii) of
App.~\ref{app:limits}.

\paragraph{Mechanisms behind the mean collapse (detail for \S\ref{sec:mech}).} The collapse
implicates two candidate contributors to the observed retention loss. \emph{Truncated reverse
diffusion}: mean-gap retention is
$19\%$ at $S{=}20$ steps, $30\%$ at our preregistered $S{=}50$ and $76\%$ at $S{=}100$.
\emph{A starved conditioning class}: oversampling crisis $15\times$ lifts retention
$30\%\to53\%$. The reported $\eta$ sweep does not support the tail-weighted loss as the primary
driver of the retention loss: retention is flat in $\eta$ and is
$29\%$ even at $\eta{=}0$. Repairing both mechanisms raises retention $47\%\to186\%$ within
the repair arm and moves the paired Sharpe by $+0.075\pm0.174$. Arms are trained separately,
so unrepaired retention differs slightly across them ($29\%$ deployed, $30\%$ step-sweep,
$47\%$ repair); we always quote an arm against its own baseline.

\paragraph{The retention quantities, defined.} Three distinct quantities appear above and
are easy to conflate, so Table~\ref{tab:retention} defines each one. All are computed
with the true regime given to the generator; none is a decision metric.

\begin{table}[tbp]\centering\small
\caption{\textbf{Retention metrics.} $\Delta\mu^{\mathrm{true}}$ is the true bull-minus-crisis
conditional mean vector and $\Delta\hat\mu$ the generator's. The first row is the headline
$29\%$; the middle rows are the same scalar computed in other arms; the last two rows are
the \emph{directional} and \emph{vector} diagnostics that show why the scalar is
insufficient.}
\label{tab:retention}
\begin{tabular}{llll}
\toprule
Quantity & Definition & Reported values & Arm \\
\midrule
Mean-gap retention & $\|\Delta\hat\mu\|_1/\|\Delta\mu^{\mathrm{true}}\|_1$ & $29\%$ ($90$/$313$\,bps) & deployed pipeline \\
\quad same, step sweep & as above & $19$/$30$/$76\%$ ($S{=}20$/$50$/$100$) & step-count sweep \\
\quad same, repair arm & as above & $47\%\to186\%$ & repair arm \\
Directional similarity & $\cos(\Delta\hat\mu,\Delta\mu^{\mathrm{true}})$ & $0.59\to0.97$ & repair arm \\
Size ratio & $\|\Delta\hat\mu\|_2/\|\Delta\mu^{\mathrm{true}}\|_2$ & $208\%$ (overshoot) & repair arm \\
Relative vector error & $\|\Delta\hat\mu-\Delta\mu^{\mathrm{true}}\|_2/\|\Delta\mu^{\mathrm{true}}\|_2$ & $0.82\to1.18$ & repair arm \\
\bottomrule
\end{tabular}
\end{table}

\section{End-to-end comparison at matched risk}\label{app:mainperf}
This table supports the claim that end-to-end performance alone is insufficient. It is not a
performance result in its own right, which is why it sits here (\S\ref{sec:hyp}).

\begin{table}[tbp]\centering\small
\caption{\textbf{Does the complete system deliver genuine improvement at matched risk?}
(\textbf{exploratory}, 5 seeds; $\pm$ is the across-seed standard deviation, not a paired
contrast; those are in Tables~\ref{tab:intervene}--\ref{tab:factorial}.) The static frontier is
the reference: ``Excess'' is Sharpe relative to the frontier at the same realised MaxDD, the
comparator being the best static blend within the restricted EW--MinVar family. \arc\ does not
outperform the frontier, whereas the regime-conditional bootstrap (ten lines, no generative
model) does, which is what shows the environment contains usable regime information. The
$-$CVaR ablation lies beyond the frontier's risk endpoint (off-supp.), as does mean--variance
(Table~\ref{tab:restbase}).}
\label{tab:perf}
\setlength{\tabcolsep}{4.5pt}
\begin{tabular}{lccccc}
\toprule
Method & Sharpe & MaxDD\% & Turn.\% & Frontier & Excess \\
\midrule
Equal weight            & $0.74\pm0.46$ & $18.5$ & 0.0 & 0.74 & $+0.00$ \\
Minimum variance        & $0.37\pm0.32$ & $\mathbf{8.5}$ & 2.8 & 0.38 & $-0.01$ \\
Gaussian $\to$ QP       & $0.41\pm0.27$ & $12.1$ & 19.7 & 0.60 & $-0.19$ \\
Block bootstrap $\to$ QP& $0.44\pm0.28$ & $10.3$ & 19.6 & 0.52 & $-0.08$ \\
\textbf{Regime-cond.\ bootstrap $\to$ QP} & $\mathbf{0.85\pm0.10}$ & $10.9$ & 19.6 & 0.53 & $\mathbf{+0.32}$ \\
\midrule
\textbf{\arc\ (full)}   & $0.56\pm0.29$ & $10.5$ & 18.1 & 0.60 & $-0.04$ \\
\quad $-$ verifier      & $0.41\pm0.29$ & $9.7$  & 18.3 & 0.48 & $-0.07$ \\
\quad $-$ CVaR term     & $0.84\pm0.50$ & $19.9$ & 18.4 & \emph{off-supp.} & n/a \\
\bottomrule
\end{tabular}
\end{table}

\section{Remaining baselines}\label{app:restbase}
Table~\ref{tab:restbase} reports the five implemented baselines not shown in
Table~\ref{tab:perf} or Table~\ref{tab:adv}, so that all eleven are accounted for. They are
run under the same box, turnover cap and costs as every other method, on the same paths.
Like Table~\ref{tab:perf} these are \textbf{exploratory} five-seed results and are not
paired contrasts.

\begin{table}[tbp]\centering\small
\caption{The five remaining implemented baselines (\textbf{exploratory}, $5$ seeds; mean
$\pm$ across-seed standard deviation). None reaches the regime-conditional bootstrap's
$0.85$. Mean--variance attains the highest Sharpe of the five but at $19.2\%$ drawdown,
beyond the static frontier's risk endpoint, the same risk-axis confound the frontier test
exists to expose, and the reason it is not evidence against our conclusion.}
\label{tab:restbase}
\begin{tabular}{lccc}
\toprule
Method & Sharpe & MaxDD\% & Turn.\% \\
\midrule
Risk parity                & $0.52\pm0.30$ & $11.4$ & $1.5$ \\
Mean--variance             & $0.71\pm0.50$ & $19.2$ & $12.0$ \\
Hierarchical risk parity   & $0.36\pm0.29$ & $8.7$  & $3.2$ \\
Historical CVaR            & $0.45\pm0.29$ & $10.1$ & $19.6$ \\
DRO-CVaR                   & $0.34\pm0.27$ & $9.7$  & $19.9$ \\
\bottomrule
\end{tabular}
\end{table}

\section{Oracle hierarchy}\label{app:oracle}
The hierarchy of \S\ref{sec:oracle} in full. All four oracles have perfect knowledge of
the latent regime and run under identical box, turnover and cost settings; they differ
only in the action set they are permitted to express that knowledge through. Figure
\ref{fig:oracle} plots them against the static EW/MinVar frontier, so the headroom is
read as a \emph{vertical} gap at matched drawdown.

\begin{figure}[!tb]
\centering
\includegraphics[width=0.78\linewidth]{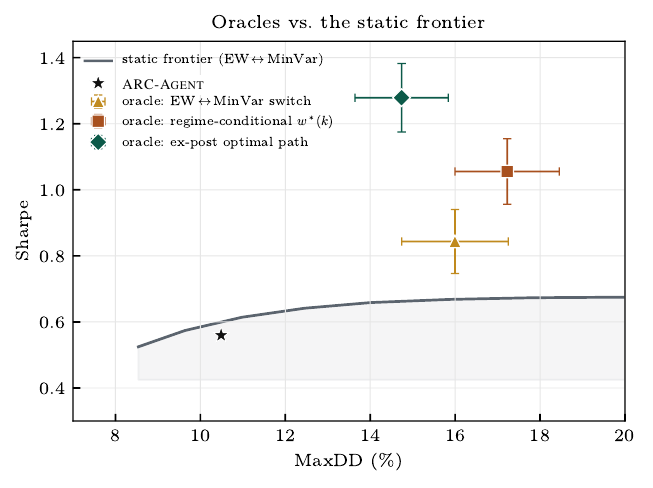}
\caption{Oracles versus the static EW/MinVar frontier ($24$ seeds; bars are $\pm1$
s.e.; confirmatory, paired under common random numbers; Sharpe at $rf{=}1.6\%$). Construct:
attainable decision value of perfect latent-state information under each action representation.
The wider L2 and L3 action representations exhibit positive matched-risk oracle headroom in U3.
\arc\ lies close to the evaluated static frontier under the reported matched-risk comparison.
Unresolved: whether this headroom transfers beyond U3 or beyond the four action sets shown.}
\label{fig:oracle}
\end{figure}

\begin{table}[tbp]\centering\small
\caption{Oracle hierarchy ($24$ seeds, paired). \textbf{All Sharpe levels in this table are
computed at the preregistered $rf{=}1.6\%$}, not at the $rf{=}0$ used for the primary U3
metrics (App.~\ref{app:metrics}); the paired L2$-$L0 contrast is $+0.28\pm0.08$ here and
$+0.248$ at $rf{=}0$ (\S\ref{sec:oracle}). ``Excess'' is Sharpe over the static
EW/MinVar frontier \emph{at the same drawdown}, a different quantity from that paired
contrast. The action set decides the answer: L1 is
near-tautological, while L2, the oracle appropriate to a regime-conditional
allocator, shows positive paired headroom under the wider action representation.}
\label{tab:oracle}
\begin{tabular}{llccc}
\toprule
& Oracle over\ldots & Sharpe & MaxDD\% & Excess \\
\midrule
L0 & (no regime knowledge) & $0.78$ & $12.4$ & $0.00$ \\
L1 & two-point action set (EW$\leftrightarrow$MinVar) & $0.84$ & $16.0$ & $+0.18$ \\
L2 & \textbf{regime-conditional weights $w^\star(k)$} & $\mathbf{1.06}$ & $17.2$ & $\mathbf{+0.38}$ \\
L3 & ex-post optimal weight path & $1.28$ & $14.7$ & $+0.62$ \\
\midrule
& \textbf{\arc} & $0.56$ & $10.5$ & $\mathbf{-0.04}$ \\
\bottomrule
\end{tabular}
\end{table}

\section{Red-team results}\label{app:redteam}

Full numbers for the four \emph{executed} attacks summarised at verification-surface level
in Table~\ref{tab:redteam} of \S\ref{sec:adversarial}. A5 (constraint injection
against the validator) is registered and was not run, so it has no row here and no claim
is made about it. The two comparison rows at the foot of the table are the rule-based
weights-emitting controls of App.~\ref{app:metrics}, standing in for an agent that emits weights
directly with no solver to configure; they are \emph{not} language-model agents, and no LLM
was executed anywhere in this work.

\begin{table}[tbp]\centering\small
\caption{Red-team results, \textbf{measured} ($5$ seeds; $\pm$ is the across-seed
standard deviation, not a paired contrast). Feasibility and provenance survive every
\emph{executed} attack (A1--A4), but so does feasibility for an agent that merely
projects, which is why the surviving claim concerns auditability.
\textbf{``Infeas.'' is a post-hoc audit, not a construction guarantee}: every committed
weight vector is checked against the budget ($10^{-4}$), box ($10^{-6}$) and turnover
($10^{-4}$) tolerances, and the column counts the rebalances that fail
(App.~\ref{app:spec}). \textbf{``Solver record''} reports whether the underlying convex
solve produced a complete active-set and KKT record; it does \emph{not} imply that the
post-processed committed vector is itself KKT-optimal.
\textbf{A5 (validator injection) was not executed and has no row.} The two comparison
rows are rule-based weights-emitting controls; no LLM was executed.}
\label{tab:adv}
\begin{tabular}{llccccc}
\toprule
& Attack & Sharpe & MaxDD\% & Infeas. & Solver record \\
\midrule
\multirow{5}{*}{\rotatebox{90}{\arc}}
& none (base)               & $0.56\pm0.29$ & $10.5$ & \textbf{0} & \textbf{100\%} \\
& A1 corrupted controller   & $0.50\pm0.41$ & $12.4$ & \textbf{0} & \textbf{100\%} \\
& A2 poisoned memory        & $0.52\pm0.28$ & $10.4$ & \textbf{0} & \textbf{100\%} \\
& A3 conformal-gate bypass& $0.55\pm0.29$ & $9.9$  & \textbf{0} & \textbf{100\%} \\
& A4 adversarial controller (injection proxy)      & $0.83\pm0.38$ & $16.7$ & \textbf{0} & \textbf{100\%} \\
\midrule
\multicolumn{2}{l}{weights-emitting agent (no solver)}   & $0.50\pm0.42$ & $25.7$ & \textcolor{rust}{\textbf{345}} & \textcolor{rust}{\textbf{0\%}} \\
\multicolumn{2}{l}{weights-emitting agent $+$ projection} & $0.72\pm0.40$ & $19.4$ & \textbf{0} & \textcolor{rust}{\textbf{0\%}} \\
\bottomrule
\end{tabular}
\end{table}

\section{Why scalar local signals are insufficient (detail for \S\ref{sec:mech})}\label{app:scalars}
Does any single upstream score verify downstream decision quality? Ours does not. Separate
step-count and crisis-oversampling sweeps implicate two candidate
contributors to the observed retention loss: truncated reverse diffusion and limited
representation of the crisis conditioning class. The combined repair targets both contributors
and therefore does not isolate their individual downstream effects; the reported $\eta$ sweep
does not support the tail-weighted loss as the primary driver (App.~\ref{app:factorial}).
Repairing both raises retention $\mathbf{47\%\to186\%}$ while the paired Sharpe moves
$\mathbf{+0.075\pm0.174}$, with an interval spanning zero; the repaired generator ($0.63$) still
loses to the bootstrap ($0.85$). The vector diagnostics show why: the repair fixes the
\emph{direction} of the regime mean-shift ($\cos 0.59\to0.97$) but \emph{overshoots} its size
($208\%$), so the relative error of the mean vector gets \emph{worse} ($0.82\to1.18$).
Distributional fidelity, conditional mean retention, verifier pass rate and downstream utility
are not interchangeable, and a module can pass a local verification metric while failing the
system-level objective. Our own measurements do not support \emph{decision-functional retention}
as the missing diagnostic either. This is the decision-focused-learning critique
\citep{elmachtoub2022spo,donti2017task} in sharp form, and it constrains what a remedy must
target: not a better distributional score --- likelihood, the verifier's calibration tests, or
even the conditional mean the optimiser literally reads --- but the decision itself
(App.~\ref{app:related}).

\section{Preregistered verdicts in full (detail for \S\ref{sec:hyp})}\label{app:verdicts}
Table~\ref{tab:perf} gives the comparison for the arms that isolate scenario quality, and all
five preregistered verdicts follow, each against its frozen failure condition.
\textbf{H1 (tail control): inconclusive, not refuted.} \arc's MaxDD ($10.5$) versus a block
bootstrap ($10.3$) is a paired $+0.2$pp against a design with little power at that scale;
what we \emph{can} say is that the diffusion model shows no detectable benefit even with
perfect conditioning. \textbf{H1b (does the CVaR term do the work?): supported, with a
twist.} Removing it inflates MaxDD $10.5\to19.9$ \emph{and} raises Sharpe to $0.84$, the
highest in the table: a drawdown reduction bought with a real Sharpe cost, precisely the
risk-axis trade the frontier test exposes. \textbf{H2 (union-bound correction): not
exercised.} The correction never binds, so nothing here tests Prop.~\ref{prop:union}.
\textbf{H3 (containment of corrupted intent): not supported} (\S\ref{sec:adversarial}).
\textbf{H4: not executed.} No LLM controller was run.

\section{Active-set analysis from the logged duals}\label{app:duals}
\paragraph{Which constraints actually bind (using the logged duals).} Provenance is
worth more than a $100\%$ checkbox only if the KKT record is used. We therefore
report active-set frequencies from the logged duals: in the illustrative run the
turnover cap is slack in calm regimes (mean $\|\Delta w\|_1 \approx 8\%$ vs.\ the
$20\%$ cap) but \emph{binds in $\sim$$38\%$ of crisis-onset rebalances}, when the
agent wants to de-risk faster than the cap allows; the box binds on the volatility
proxy and short-duration sleeves during flights-to-quality. This is the one
genuinely informative thing provenance yields, and it also sharpens the case for
executing A5: the feasibility guarantee is not guarding an idle door. The turnover
constraint is exactly what an attacker (or a panicked controller) would most want to
relax, and it is the one the validator refuses to expose. A5 remains registered and
unexecuted (\S\ref{sec:adversarial}).

\paragraph{What the record does and does not give.} The KKT record certifies feasibility
and stationarity of the \emph{solver output} $w^{\mathrm{sol}}$ for the problem it was
solved against. The subsequently committed action $w_t$, obtained from $w^{\mathrm{sol}}$
by clipping and renormalisation (App.~\ref{app:spec}), \emph{passed the specified post-hoc
feasibility audit}, but we do not claim it retains the solver's KKT optimality certificate
(Prop.~\ref{prop:feas}). Neither claim, by itself, re-derives $w_t$ from
scratch. \emph{Deterministic replay} additionally
requires the realised scenario matrix (or the seed and generator state that produced it),
the model checkpoint hash, the solver name, version and tolerances, and the tie-breaking
rule used when the active set is degenerate. Our log records these alongside the
certificate, but the guarantee we state and prove is the verification claim above; replay
is an implementation property of the release, not a theorem.

\section{Boundaries of established knowledge}\label{app:limits}
This appendix is organised by \emph{claim} rather than by experiments that ``should'' be run.
Each entry states the exact limit of the existing evidence, and none is a promise to fix it for
this paper. With the claim--evidence ledger (Table~\ref{tab:ledger}), these entries record
what the audit did not establish.

\paragraph{1. Domain boundary.} Every empirical statement is scoped to U3 and to
one decision architecture. We do not claim ``generative modelling is a bottleneck in
modular decision systems'' in general; we claim that \emph{this} conditional
diffusion model, trained this way, forfeits the conditional mean on \emph{this} benchmark,
that the repair arm targets two plausible contributors to the collapse, and that repairing them does not
repair the decision. The regime-conditional bootstrap's advantage is partly favoured by
U3's stationary within-regime conditionals: a resampler inherits the conditional law by
construction, so it is close to a near-oracle generator here. Establishing where conditional
\emph{generation} could pay requires sweeping within-regime nonstationarity, regime overlap
and duration heterogeneity, which we have not done.

\paragraph{2. Masking boundary.} Perception cannot be judged independently while the
generator ignores its conditioning. Perfect detection does not help \emph{through the
current generator}; that does not establish perception is unimportant, and perception may
matter once the generator is fixed. The factorial detector $\times$ generator matrix that
would separate the two was not run, so the verdict in Table~\ref{tab:ledger} is
\emph{unresolved}: neither supported nor refuted.

\paragraph{3. Intervention boundary.} The stage interventions
(Table~\ref{tab:intervene}) and the channel factorial (Table~\ref{tab:factorial}) use
different baselines and seed counts and were trained separately, so the difference between
the full generator bypass and the all-three-channel repair \emph{cannot be decomposed}
(\S\ref{sec:masked}, App.~\ref{app:factorial}). Rerunning both against one baseline under
common random numbers at a single seed count would identify whether the three channels exhaust
the bypass.

\paragraph{4. Risk-axis boundary.} Table~\ref{tab:factorial} reports $\Delta$Sharpe alone:
the factorial lacks MaxDD and turnover per cell, so we cannot show that a channel repair
moves a portfolio \emph{off} the risk frontier rather than along it, the confound our own
frontier test exists to expose and which we apply everywhere else. The risk-axis position of
each factorial cell is therefore not identified by the reported runs.

\paragraph{5. Architecture boundary.} One conditional generator family is evaluated.
Without conditional VAE, normalising-flow, plain conditional diffusion and
moment-regularised variants, the localisation result remains single-architecture.

\paragraph{6. Auditability boundary.} The surviving governance claim is an \emph{argument
from the available records}, not a measurement: we have not scored whether an auditor can
recover the binding constraint and the reason-for-trade from the log alone. This is the
paper's weakest link, since it is the one claim asserted without evidence.

\paragraph{7. Controller boundary.} No LLM controller was executed. All measured runs use a
\textbf{rule-plus-contextual-bandit controller}; the system under test is a planning loop
over typed tool configurations, and A4 is an adversarial-\emph{controller} proxy, not a live
prompt injection. Nothing here bears on language-model reasoning or prompt-injection
robustness.

\paragraph{8. Statistical boundary.} Seed counts are $5$--$24$ and inconsistent across
tables. The factorial is the binding case: at six paired seeds the mean-only interval is
very wide and the interaction interval includes zero, so the \emph{magnitude} of the
mean--scenario interaction is unresolved and the headline $55\%$ is a point estimate.
Twenty to thirty paired seeds under common random numbers would settle it, and
App.~\ref{app:cost} records that compute is not what is stopping us.

\paragraph{9. Certificate boundary.} The implemented one-step score does not identify
$\CVaR_\alpha$ (\S\ref{sec:theory}, App.~\ref{app:proofs}), so the executed gate is an
operational filter and not a validated decision-quality certificate. H2 was never exercised,
so nothing in this paper tests the refinement correction either.

\paragraph{10. Threat-model boundary.} Trusted-path compromise was not executed. A5 targets
the validator, and the solver, post-solve cleanup and audit code are equally inside the
trusted path; robustness to compromise of any of them is unevaluated, so the feasibility
result of Table~\ref{tab:redteam} is evidence about attacks that leave those components
alone.
\paragraph{Future work.} Decision-focused training of the scenario generator; a
disaggregated retention analysis over the mean, covariance and scenario tail/dependence channels; a second
decision domain (inventory, energy dispatch) to test whether the localisation protocol
transfers; and conformal \emph{risk control} on a monotone loss in place of coverage of a
CVaR difference.

\section{Extended related work}\label{app:related}
\paragraph{Agentic systems and tool use.}
ReAct-style interleaving of reasoning and acting \citep{yao2023react} and
code-as-action agents \citep{wang2024codeact} motivate our typed configuration
interface. Recent work couples language models with formal optimisation:
autoformulation of optimisation models \citep{astorga2025autoformulation} and
generative simulation with gradient-free calibration \citep{holt2025gsim} show
that LLMs are most reliable when they \emph{configure} a solver rather than
replace it. Lightweight online memory without finetuning \citep{holt2025memory}
informs our regime-indexed episodic store. We adopt this ``LLM plans, solver
acts, verifier gates'' pattern as a governance principle rather than a
convenience. A systems literature instruments running systems at stable boundaries rather than
inside volatile application code \citep{zheng2025agentsight,dong2024agentops,koohestani2025agentguard};
our validator, feasibility audit and KKT records are instances of that discipline, and the
distinction we insist on is that \emph{observability provides evidence; verification specifies
what property that evidence supports}.

\paragraph{Regimes and structured generative time series.}
HMMs capture structural breaks and state dependence
\citep{hamilton1989,kim1994,ang2012}; we reuse their posteriors as conditioning
signals. GAN/VAE and diffusion generators improve realism
\citep{yoon2019timegan,desai2021timevae,rasul2021,tashiro2021csdi}, but are
typically unconditioned and loosely coupled to decisions; \arc instead couples a
regime-conditioned, tail-weighted diffusion generator to the allocator as a
callable, verifiable tool.

\paragraph{Decision-focused learning.} A parallel literature learns predictive
models \emph{through} the optimiser rather than upstream of it: smart
``predict-then-optimize'' \citep{elmachtoub2022spo}, task-based end-to-end learning
\citep{donti2017task}, and differentiable optimisation layers
\citep{amos2017optnet,agrawal2019cvxpylayers}, with portfolio applications
\citep{butler2023} and, more recently, decision-focused \emph{generative} models
\citep{wang2025gendfl,zhao2026diffusiondfl,xie2026dwfm}, which target exactly the
generator-to-decision coupling that our \S\ref{sec:mech} measurement finds broken. \arc\ is deliberately \emph{not} end-to-end: it keeps the
solver a fixed, certified, non-differentiated component so that its guarantees
survive a compromised upstream, and treats decision-aware training as future work. The trade-off
this involves is easy to overstate. It is \emph{not} optimality versus auditability: an upstream
model can be trained through a differentiable
convex layer \citep{amos2017optnet,agrawal2019cvxpylayers} and still be deployed against the
same fixed solver, which still emits the same KKT record, so auditability survives
decision-focused training. What decision-focused training actually costs is
\emph{modularity}: the generator becomes coupled to a particular downstream program and
to its training pipeline, and can no longer be swapped, certified, or attacked in isolation
in the way \S\ref{sec:threat} relies on. That coupling is the central design choice we are
making, and this literature is the alternative branch. Our own
measurement (\S\ref{sec:mech}) is an argument for paying that cost.

\paragraph{Distribution-free uncertainty quantification.} The verifier is implemented
as a conformal-style \emph{operational evidence filter}. Conformal prediction
\citep{vovk2005} and conformal risk control \citep{angelopoulos2023} provide
finite-sample, distribution-free coverage guarantees and are the right foundation
for a valid act-time certificate, but the implemented score does not satisfy the
identification conditions required for that interpretation (\S\ref{sec:theory}).

\paragraph{Robust allocation and governance.}
Distributionally robust and CVaR formulations address fat tails
\citep{rockafellar2000,delage2010} but hinge on scenario quality; we feed a
regime-conditioned scenario set through an operational evidence filter into a convex allocator with
Ledoit--Wolf shrinkage \citep{ledoit2004}. Multi-agent coordination
\citep{lowe2017} inspires our role-based tools, but our emphasis is on
\emph{auditability}: the allocator's KKT system provides solver-side audit trails
that align with the reliability/governance goals of structured-data agents.

\section{The U3 benchmark}\label{app:u3}
U3 is a controllable synthetic regime-shift generator over $d$ correlated instruments.
A latent Markov chain with $K{=}3$ states (\emph{bull}, \emph{stagnation},
\emph{crisis}) drives a state-dependent return law: per-state mean and covariance, with
crisis drawn from a fat-tailed, left-skewed law whose lower-tail dependence exceeds its
average correlation, so diversification degrades exactly when it is needed
(Fig.~\ref{fig:u3}b,c). There are $d{=}10$ instruments, three of them defensive. Returns
follow a regime-dependent factor model driven by a skew-$t$ market factor with shared
mixing, so no regime is Gaussian: degrees of freedom fall from $12$ (bull) through $7$
(stagnation) to $3.5$ (crisis) and skew steepens from $0$ to $-0.80$
(App.~\ref{app:spec}). Crisis is rare and persistent: the stationary crisis share is
$4.79\%$ and the mean crisis episode is about $45$ trading days, longer than the $21$-day
rebalance interval; the evaluated $1{,}500$-day development sequence realises $66$
crisis days ($4.4\%$). Four knobs are exposed: crisis \emph{frequency} (the chain's entry
probability), \emph{severity} (the crisis mean and tail index), \emph{duration} (its
self-transition probability), and \emph{detectability} (\texttt{regime\_separation},
which interpolates the per-state means toward a common value). At
\texttt{regime\_separation}$=0$ the states are indistinguishable from returns alone and
HMM balanced accuracy equals chance, $1/3$, by construction; this is the calibration
check in Fig.~\ref{fig:u3}(d). Because the latent state is recorded, regime inference can
be scored directly and the oracle hierarchy of \S\ref{sec:oracle} can be computed
exactly, neither of which is possible on historical data.

\paragraph{Horizon and splits.} Each seed draws a single $2{,}200$-day path. The first
$756$ days ($\approx$3 years) are a \textbf{warm-up} used only to initialise the HMM and to
fit the first generator; no trade is placed inside it. The remaining $1{,}444$ days are
evaluated walk-forward with a rebalance every $21$ trading days, giving \textbf{$69$
evaluated rebalances per seed}. Estimation is strictly causal: at each rebalance the
generator and detector see only data up to that date, and the generator is refit every
$1{,}000$ trading days ($2$ refits per run), the cadence behind every reported number
(Table~\ref{tab:repro}, App.~\ref{app:spec}). The generator, its default
parameter values, and the seed registry are specified in Table~\ref{tab:repro},
Table~\ref{tab:seeds} and App.~\ref{app:spec}; code and configuration files will be released upon acceptance.

\paragraph{Reproducibility summary.} Table~\ref{tab:repro} collects every setting needed
to reproduce the measured results in one place; App.~\ref{app:spec} gives the full
specification.

\begin{table}[tbp]\centering\small
\caption{\textbf{Reproducibility summary} for the measured (U3) experiments. All values are
taken from the frozen preregistration and the implementation; the data-generating process,
architecture and verifier are given in full in App.~\ref{app:spec}.}
\label{tab:repro}
\begin{tabular}{ll}
\toprule
Setting & Value \\
\midrule
\multicolumn{2}{l}{\emph{Benchmark}}\\
Regimes $K$ & $3$ (bull, stagnation, crisis); $K\in\{2,3,4\}$ searched \\
Instruments $d$ & $10$, of which $3$ defensive (negative crisis beta) \\
Transition matrix & Eq.~\eqref{eq:transition}; stationary $(0.669,0.283,0.048)$ \\
Regime persistence & $(0.992,\,0.980,\,0.978)$; mean crisis episode $\approx45$ days \\
Factor drift (annualised) & $(0.20,\,0.01,\,-0.55)$ \\
Factor daily volatility & $(0.0060,\,0.0095,\,0.0210)$ \\
Skew-$t$ degrees of freedom & $(12,\,7,\,3.5)$; skew $(0,\,-0.25,\,-0.80)$ \\
Residual correlation & $(0.10,\,0.15,\,0.30)$; idiosyncratic vol $0.0035$ \\
Crisis prevalence & $4.79\%$ stationary ($71.9$ crisis days per $1{,}500$) \\
Simulated horizon & $2{,}200$ days per seed: $756$-day warm-up, $1{,}444$ evaluated \\
Rebalances & $69$ (monthly, every $21$ trading days); embargo $21$ days \\
\midrule
\multicolumn{2}{l}{\emph{Generator and denoiser}}\\
Formulation & variance-preserving diffusion; predicts $\varepsilon$ \\
Denoiser & two experts, each $45\!\to\!128\!\to\!128\!\to\!10$ with SiLU \\
Gate & $g_t=\sigma(\mathrm{MLP}(\pi_t))$, $3\!\to\!32\!\to\!1$ with SiLU \\
Input dimension & $45 = d\,(10) + $ time embedding $(32) + K\,(3)$ \\
Parameter counts & $23{,}690$ per expert $+\,161$ gate $=47{,}541$ total \\
Diffusion steps $S$ & $50$; swept over $\{20,50,100\}$ \\
Noise schedule & linear $\beta$ from $10^{-4}$ to $0.06$ over $S$ steps \\
Scenarios $N$ & $\{256,512\}$ searched, $512$ default; $4{,}096$ in the sampling intervention \\
Return winsorisation & $|r|\le0.10$ (bounds $M$ in the concentration bound) \\
\midrule
\multicolumn{2}{l}{\emph{Training}}\\
Schedule & $12$ epochs per refit; batch $128$; Adam; lr $2\times10^{-3}$ \\
Refit cadence & every $1{,}000$ trading days ($2$ refits per run) \\
Selection & deflated Sharpe, computed within the walk-forward \\
\midrule
\multicolumn{2}{l}{\emph{Verifier}}\\
Method & adaptive conformal inference \citep{gibbs2021,zaffran2022} \\
Target miscoverage & $\delta=0.10$ (90\% long-run act-time coverage) \\
Budget & $B=3$; per-attempt level $\delta/B\approx0.033$ \\
Score $s_t$ & Definition~\ref{def:score}; $|$realised tail loss $-$ scenario $\CVaR_\alpha|$ \\
Calibration window & $60$ scores (rolling); ACI step size $0.01$ \\
Fallback & stationary block bootstrap into the same CVaR-QP \\
\midrule
\multicolumn{2}{l}{\emph{Decision and evaluation}}\\
Constraints & long-only, fully invested; box $[0,0.25]$; turnover $\le0.20$ per rebalance \\
Costs & $8$\,bps proportional \\
CVaR level $\alpha$ & $0.95$ default; $\alpha\in\{0.90,0.95,0.99\}$ searched \\
Risk-free rate & $rf{=}0$ primary on U3; oracle hierarchy at $rf{=}1.6\%$ (App.~\ref{app:metrics}) \\
Solver & CLARABEL via CVXPY ($\ge$1.4), SCS fallback; no overrides set \\
Solver tolerances & CLARABEL gap/feas.\ tol.\ $10^{-8}$, \texttt{max\_iter}$=200$; \\
 & SCS $\varepsilon_{\mathrm{abs}}{=}\varepsilon_{\mathrm{rel}}{=}10^{-4}$, \texttt{max\_iters}$=10^{5}$ \\
Hyperparameter grid & $2{,}916$ configurations $\times\,5$ seeds $=14{,}580$ (App.~\ref{app:trials}) \\
Ablation arms & $-$verifier, $-$memory, $-$MoE, $-$CVaR term, $B{=}1$ \\
Seed registry & Table~\ref{tab:seeds} \\
\bottomrule
\end{tabular}
\end{table}

\section{Full experimental specification}\label{app:spec}
\paragraph{Regime chain.} The latent chain over $\{$bull, stagnation, crisis$\}$ is built
from per-regime persistences $(0.992,0.980,0.978)$, with the exit mass split
$0.93/0.07$ from bull into stagnation/crisis, $0.88/0.12$ from stagnation into
bull/crisis, and $0.35/0.65$ from crisis into bull/stagnation. This yields
\begin{equation}
P=\begin{pmatrix}
0.99200 & 0.00744 & 0.00056\\
0.01760 & 0.98000 & 0.00240\\
0.00770 & 0.01430 & 0.97800
\end{pmatrix},
\label{eq:transition}
\end{equation}
whose stationary distribution is $(0.669,\,0.283,\,0.048)$. Crisis self-transition
$0.978$ gives a mean crisis episode of about $45$ trading days, deliberately longer than
the $21$-day rebalance interval, since an episode shorter than the decision interval could
not be rotated into by a monthly agent and the benchmark could not test the paper's claim.
The stationary crisis share of $4.79\%$ implies $71.9$ crisis days in a $1{,}500$-day
window; the evaluated path realises $66$ ($4.4\%$), an ordinary finite-path deviation.
Realised and stationary prevalence need not coincide, and it is the realised figure that
determines how much crisis data the generator actually saw.

\paragraph{Return law.} Returns follow a regime-dependent factor model,
$r_i = \mu_i(k) + \beta_i(k)F(k) + \epsilon_i$, where $F(k)$ is a skew-$t$ market factor
with a \emph{shared} chi-square mixing variable, which is what produces lower-tail
\emph{dependence} rather than merely high correlation: assets co-crash. The factor has
annualised drift $(0.20,\,0.01,\,-0.55)$, daily volatility
$(0.0060,\,0.0095,\,0.0210)$, degrees of freedom $(12,\,7,\,3.5)$ and skew
$(0,\,-0.25,\,-0.80)$ across bull, stagnation and crisis; residual cross-correlation among
risky names is $(0.10,\,0.15,\,0.30)$ and idiosyncratic volatility is $0.0035$ daily.
No regime is Gaussian: the tails thicken and the skew steepens monotonically from
bull to crisis. Of the $d{=}10$ instruments, $3$ are defensive, with near-zero normal beta
and \emph{negative} crisis beta (a flight-to-quality corner), and a lower unconditional
carry ($0.8\%$ annual) than the risky names. Both properties are load-bearing: without a
defensive corner the allocator is always fully invested in statistically identical assets,
equal weight is optimal by construction, and regime-conditioning has nothing to earn;
without the carry penalty, holding the defensive assets always would dominate and
regime-timing would again be worthless.

\paragraph{Denoiser.} Each of the two experts is a three-layer MLP
$45\to128\to128\to10$ with SiLU activations, where the $45$-dimensional input concatenates
the noisy return vector ($d{=}10$), a $32$-dimensional sinusoidal time embedding, and the
$K{=}3$ regime posterior. The gate is $3\to32\to1$ with a sigmoid output, combining the
experts as
$\hat\varepsilon_\theta=(1-g_t)\hat\varepsilon^{\mathrm{base}}_\theta+g_t\hat\varepsilon^{\mathrm{crisis}}_\theta$
so the weights sum to one by construction. Exact parameter counts are $23{,}690$ per
expert and $161$ for the gate, $47{,}541$ in total; the $-$MoE ablation drops the crisis
expert, leaving $23{,}851$. Diffusion uses $S{=}50$ steps with $\beta$ linear from $10^{-4}$ to $0.06$. For the
reported baseline and ablation table the generator is trained for $12$ epochs per refit,
batch $128$, Adam at lr $2\times10^{-3}$, and refit every $1{,}000$ trading days, giving
\textbf{two refits per run}. \emph{Two other scripts in the released code, the smoke test
and an early ablation sweep, use a different generator budget ($25$ epochs, refit every
$504$ trading days) and are not the source of any reported number}; the values here are
those behind Table~\ref{tab:perf}, and they are the values used by every measured arm
(\texttt{diff\_epochs=12}, \texttt{retrain\_every=1000}).

\paragraph{Verifier, operationally.} The score of Definition~\ref{def:score} is computed
exactly as Definition~\ref{def:score} states. The verifier keeps a rolling window of the
$60$ most recent scores, takes the empirical $(1-\beta_t)$-quantile of that window as the
certificate width $q_t$, targets miscoverage $\delta{=}0.10$, and updates the running level
online by
$\beta_{t+1}=\beta_t+\gamma_{\mathrm{ACI}}\,(\delta-\mathbf{1}\{s_t>q_t\})$
with step size $\gamma_{\mathrm{ACI}}=0.01$ and $\beta_0=\delta/B$.
Attempts are capped at $B{=}3$ and each is evaluated at $\delta/B\approx0.033$; the first
passing attempt is accepted. If no attempt passes, or if the solve fails, the system falls
back to a stationary block bootstrap whose scenarios are solved by the \emph{same} CVaR-QP.

\paragraph{Why (C5) is not a technicality, and what would discharge it.} The gap in
(C5) is structural rather than a missing lemma. Our score is a single-realisation residual
$|\ell_{t+1}-\widehat{\CVaR}_\alpha|$, and $\CVaR_\alpha$ is \textbf{not elicitable}: no
scoring function has $\CVaR_\alpha$ as its unique minimiser \citep{gneiting2011}, so no
one-sample residual can identify it, and no amount of conformal calibration on such a score
can be made to bound $|\CVaR^{p_\theta}_\alpha-\CVaR^{\Prob}_\alpha|$. This is why we mark
(C5) unproved rather than deferred: as specified, it is not provable.

There is, however, a constructive route, and we state it because it prescribes a concrete
change to the verifier; no reported run is reinterpreted by it. The \emph{pair}
$(\mathrm{VaR}_\alpha,\CVaR_\alpha)$ \textbf{is} jointly elicitable, via the
Fissler--Ziegel class of strictly consistent scoring functions \citep{fissler2016}. A
verifier that scored the model's $(\mathrm{VaR},\CVaR)$ pair under an FZ-consistent loss
would inherit strict consistency, and under the accompanying strong-identifiability
condition an excess-score bound converts into a bound on the population discrepancy,
supplying exactly the bridge (C5) assumes. We did not run this variant, so we register it
as the repair our own diagnosis implies, in the same spirit as \S\ref{sec:attrib}:
the protocol convicts the paper's theory as well as its generator.

\paragraph{What $s_t$ is and is not.} $s_t$ is an \emph{observable} absolute one-step discrepancy
between a scenario CVaR estimate and a single realised loss. It is not an estimate of the
population quantity $|\CVaR^{p_\theta}_\alpha-\CVaR^{\Prob}_\alpha|$, and a single realised
return cannot reveal that quantity. The link between the calibrated quantiles of
$\{s_\tau\}$ and the population discrepancy is exactly the identification condition that
the conditional certificate statement~\ref{thm:certified} assumes and that we do not prove (App.~\ref{app:proofs}).

\paragraph{Allocator and solver.} The convex program \eqref{eq:qp} is solved through CVXPY
with \textbf{CLARABEL} as the primary interior-point solver and \textbf{SCS} as the
fallback, invoked only if the primary solve raises. No tolerance, iteration or
polish setting is overridden anywhere in the code, so each solver's compiled-in defaults
apply. We state them here so the PDF is self-contained: CLARABEL uses absolute and
relative duality-gap tolerances and a feasibility tolerance of $10^{-8}$ with a
$200$-iteration limit; SCS uses $\varepsilon_{\mathrm{abs}}=\varepsilon_{\mathrm{rel}}=
10^{-4}$ with a $10^{5}$-iteration limit. The released dependency specification pins
CVXPY $\ge1.4$, the first release in which CLARABEL is the default conic solver, and
CLARABEL and SCS enter as transitive dependencies of that pin. The solver reports one of three outcomes per
rebalance: a normal solve, a fallback solve, or, if no weight vector is returned,
status \texttt{FAILED}, in which case the committed action is the previous feasible
portfolio $w_{t-1}$. \texttt{FAILED} did not occur in any reported run. Returns are
winsorised at $|r|\le0.10$, which also bounds $M$ in the concentration term of
Lemma~\ref{a-lem:conc}.

\paragraph{Post-solve cleanup: what it does and does not guarantee.} The solver's solution
$w^{\mathrm{sol}}$ is clipped to the box and renormalised to sum to one before it is
committed. This is numerical cleanup and a \emph{no-op at an exact solution}, but it is
post-processing, and we state its guarantee precisely because an earlier draft overstated
it. Clipping alone gives $\ell\le w_i\le u$. The subsequent division by $\sum_i w_i$
restores the budget exactly and, with $\ell{=}0$ as here, preserves non-negativity, but it
is \textbf{not} feasibility-preserving in general. If the clipped vector sums to $s<1$ every
coordinate is scaled up by $1/s$, which can lift one above $u$; for $\ell>0$ and $s>1$ the
symmetric violation is possible at the lower bound; and rescaling changes
$\|w_t-w_{t-1}\|_1$, so the turnover cap can be breached too. Nor do we claim a general
perturbation bound. Renormalisation rescales \emph{every} coordinate by $1/s$, so the
displacement it induces grows with the dimension and with how close the active coordinates
sit to their bounds; bounding it would need assumptions we have not stated. What we rely on
is that the solver residuals are numerically small, so
the post-processing perturbation was small in the reported runs, and, decisively, on the
explicit audit. No theoretical guarantee is doing any work here. The paper's feasibility
claim is therefore checked, not constructed: the ``Infeas.'' column of Table~\ref{tab:adv} is a post-hoc audit
of every committed weight vector against the budget ($10^{-4}$), box ($10^{-6}$) and
turnover ($10^{-4}$) tolerances, and it is zero for every executed arm and seed. That is
evidence, not a guarantee. ``Verifiable provenance'' likewise means the solve is
reconstructible from the log, not that a certificate has been re-derived for a
post-processed vector: the logged KKT duals certify $w^{\mathrm{sol}}$, not the
clipped iterate.

\paragraph{The variant that would restore the guarantee (registered, not run).} The clean
design does not post-process at all. The trusted validator should check the budget, box and
turnover residuals of $w^{\mathrm{sol}}$ against fixed tolerances and commit
$w^{\mathrm{sol}}$ only if all three pass, invoking the predefined block-bootstrap fallback
otherwise. That variant is better on two counts: every committed action is feasible to a
stated tolerance by a check the \emph{validator} owns rather than by a numerical
convenience, and the committed vector is the solver's own iterate, so the logged KKT record
certifies the action actually taken and the $w^{\mathrm{sol}}$-versus-$w_t$ distinction of
Prop.~\ref{prop:feas} disappears. A Euclidean projection onto the full constraint set would
also restore feasibility, but it is a second optimisation whose output carries no KKT record
from the first, so it reintroduces the gap it closes. We ran neither: every number in this
paper comes from the clip-and-renormalise pipeline described above, and we register the
change rather than retrofit it, exactly as we do for the (C5) scoring change above.

\paragraph{Action-space partition (detail for \S\ref{sec:threat}).} The \emph{untrusted},
controller-settable parameters are the objective and generator parameters
$(K, N, \eta, q, \lambda, \gamma, \lambda_\mu, \alpha)$; none of them can enlarge the feasible
set, so a corrupted controller can degrade decision \emph{quality} but not admissibility. The
\emph{trusted}, validator-owned parameters are the constraint parameters $(\ell,u,\tau)$ and the
refinement budget $B$. The validator clips any untrusted value to its preregistered range and
rejects any attempt to name a constraint parameter.

\paragraph{Controller.} The controller is a \textbf{transparent rule plus a contextual
bandit} (LinUCB-lite) over the preregistered grid; no LLM is involved. The rule keys off
the crisis posterior $\pi_{t,\mathrm{crisis}}$: above $0.35$ it sets
$\eta{=}3.0$, $\alpha{=}0.99$, $\gamma{=}1.5$, $q{=}0.10$ and $\lambda_\mu{=}4$ (risk-off);
between $0.15$ and $0.35$ it sets $\eta{=}2.0$, $\alpha{=}0.95$, $\lambda_\mu{=}12$; below
$0.15$ it sets $\lambda_\mu{=}24$ (risk-on). The bandit then refines
$\{\eta,\gamma,\lambda,\lambda_\mu\}$ within the grid, keeping per-$(\text{regime},
\text{knob}, \text{value})$ reward means and exploring uniformly with probability
$\varepsilon=0.15$. Every value it can emit is inside the preregistered ranges, and the
validator clips anything outside them.

\paragraph{Regime detector.} A Gaussian HMM (\texttt{hmmlearn}) with $K$ states on
$z$-scored features, \texttt{covariance\_type="diag"}, \texttt{n\_iter}$=200$,
\texttt{tol}$=10^{-4}$, and \textbf{4 random restarts} scored by log-likelihood with the
best retained; restart seeds are $\text{seed}+13r$. The offline regime-scoring used for the
detectability calibration of Fig.~\ref{fig:u3}(d) instead uses
\texttt{covariance\_type="full"} with \texttt{tol}$=10^{-3}$ and the same restart count.

\paragraph{Baseline tuning.} The scenario-fed baselines (historical CVaR, DRO-CVaR,
Gaussian, stationary block bootstrap) are \emph{untuned by construction}: they feed the
same CVaR-QP at the same $\alpha$, box, turnover cap and costs, so only scenario quality
differs. The classical baselines take their standard forms with no free parameters beyond
the shared constraints, except HRP, which uses the standard single-linkage
correlation-distance clustering, and the block bootstrap, whose block length is $21$ days
(one rebalance interval).

\paragraph{Attacks, operationally.} A1 randomises $\lambda_\mu$ and $\gamma$ within
$[0,30]$; A2 flips the HMM regime posterior; A3 bypasses the conformal gate (it does not touch the validator); A4 maximises
turnover. None of the four modifies the trusted validator, which is why the feasibility
result they support excludes validator compromise (A5, not executed).

\begin{figure}[!tb]
\centering
\includegraphics[width=\linewidth]{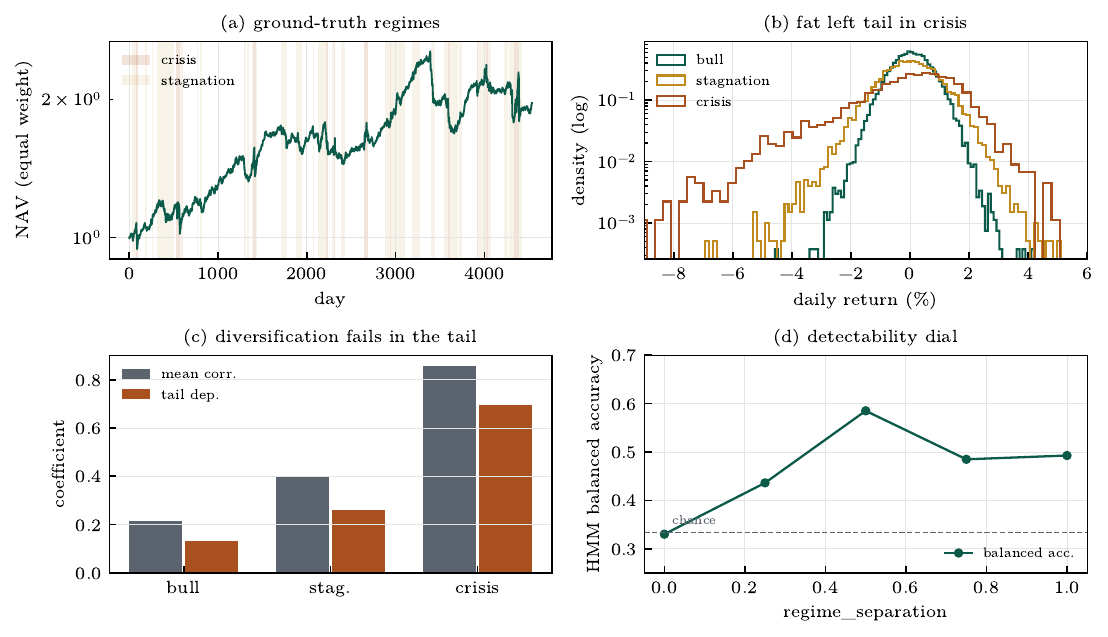}
\caption{The U3 benchmark, drawn on an \textbf{illustrative $4{,}536$-day path} from the
same generator (panel (d) uses $2{,}600$-day paths per seed); this is for visualising the
data-generating process and is \emph{not} the evaluated horizon, which is $2{,}200$ days per seed
($756$ warm-up, $1{,}444$ evaluated). \textbf{(a)} ground-truth regimes; \textbf{(b)} crisis
returns are fat-tailed and left-skewed; \textbf{(c)} correlations \emph{and}
lower-tail dependence both spike in crisis, so diversification fails exactly when it
is needed; \textbf{(d)} the detectability dial: at \texttt{regime\_separation}$=0$
balanced accuracy equals chance ($1/3$) by construction.}
\label{fig:u3}
\end{figure}

\section{Problem setup, tools, and the agent loop}\label{app:setup}
At rebalance date $t$, let $R_t \in \R^d$ be observed returns and
$\wt \in \R^d$ portfolio weights satisfying the budget, box, and turnover
constraints
\begin{equation}
\mathbf{1}^\top \wt = 1,\qquad \ell \le \wt \le u,\qquad
\|\wt - w_{t-1}\|_1 \le \tau .
\label{eq:feasset}
\end{equation}
Regimes follow a $K$-state Gaussian HMM with latent state $S_t$ and posterior
$\pi_{t,k} = \Prob(S_t{=}k \mid R_{1:t})$. All estimation is strict
walk-forward. \arc orchestrates three specialist tools, treated as fixed, trusted
backbones; the controller and evidence-filter layer (\S\ref{sec:bench}) sits around them.

\paragraph{Tool A: Regime detector.} Fit/update the HMM on $\{R_s\}_{s\le t}$ and
emit $\pi_t$ and a context vector $z_t$ encoding $\arg\max_k \pi_{t,k}$ and recent
statistics.

\paragraph{Tool B: Scenario generator.} A variance-preserving,
regime-conditioned diffusion model with a tail-weighted objective. With
$\tilde{\ell} = -\min_j r_j$ a portfolio-free adverse proxy,
\begin{equation}
\Ltail = \E\!\Big[\big(1 + \eta\, \mathbf{1}\{\tilde\ell \ge Q_{1-q}(\tilde\ell)\}\big)
\,\|\varepsilon - \varepsilon_\theta(\cdot)\|_2^2\Big],
\quad q\in[0.05,0.10],\ \eta\in[1,3],
\label{eq:ltail}
\end{equation}
The denoiser is a two-expert (mixture-of-experts) network,
$\hat\varepsilon_\theta = (1-g_t)\hat\varepsilon^{\mathrm{base}}_\theta
+ g_t\hat\varepsilon^{\mathrm{crisis}}_\theta$ whose gate
$g_t = \sigma(\mathrm{MLP}(z_t))$ increases with the crisis posterior. At
deployment it draws $N$ scenarios $\{r^{(i)}_{t+1}\}$ conditioned on $z_t$.

\textbf{Note on the indicator (corrected).} An earlier draft printed
Eq.~\eqref{eq:ltail} with $\mathbf{1}\{\tilde\ell \le Q_q(\tilde\ell)\}$, which
up-weights the \emph{least} adverse samples, the opposite of the intent. The
\emph{implementation} was always correct; the equation was not. We verified this rather
than asserting it: the up-weighted set has mean return $-107$\,bps and is $2.9\times$
enriched in crisis days, whereas the erroneous indicator would have selected samples
averaging $+106$\,bps. The audit script will be released upon acceptance.
The $29\%$
mean-collapse result is therefore not an artefact of an inverted indicator, and $\eta$
is ruled out independently (\S\ref{sec:mech}).

\paragraph{Tool C: CVaR allocator.} With per-scenario loss
$\ell_i = -w^\top r^{(i)}_{t+1}$ and blended, shrunk moments
$(\hmu_t,\hSig_t)$, solve the convex epigraph QP
\begin{align}
\min_{w,\zeta,\{s_i\}}\ & -\lambda_\mu\, \hmu_t^\top w
+ \gamma\, w^\top \hSig_t w
+ \zeta + \frac{1}{(1-\alpha)N}\sum_{i=1}^{N} s_i \label{eq:qp}\\
\text{s.t.}\ & s_i \ge 0,\quad s_i \ge \ell_i - \zeta,\quad i=1,\dots,N, \nonumber\\
& \mathbf{1}^\top w = 1,\quad \ell \le w \le u,\quad \|w - w_{t-1}\|_1 \le \tau. \nonumber
\end{align}
The objective is a convex QP (quadratic $w^\top\hSig_t w$ plus a CVaR epigraph
with $O(N)$ affine constraints); interior-point solves scale as
$O(d^3 + N d^2)$. We use the Rockafellar--Uryasev CVaR at level $\alpha{=}0.95$
and shrink $\hSig_t$ toward the identity \citep{ledoit2004}. The KKT system logs
active constraints, tail weights, and duals at each rebalance, the substrate for
governance in \S\ref{sec:threat}.

\paragraph{The agent loop.} Algorithm~\ref{alg:loop} states the plan--act--verify--refine
loop of \S\ref{sec:bench} in full. Configurations are code-as-action: the controller may
set objective and generator parameters, never weights and never constraint parameters,
which the trusted validator owns (\S\ref{sec:threat}).

\begin{algorithm}[t]
\caption{\arc: one rebalance at date $t$}
\label{alg:loop}
\begin{algorithmic}[1]
\State \textbf{input:} history $R_{1:t}$, previous weights $w_{t-1}$, refinement budget $B$
\State $\pi_t, z_t \gets \textsc{RegimeDetector}(R_{1:t})$ \Comment{Tool A; strict walk-forward}
\State $a_t \gets \textsc{Controller}(z_t, \text{memory})$ \Comment{config only; validator clips to preregistered ranges}
\State $b \gets 0$
\Repeat
\State $\{r^{(i)}_{t+1}\}_{i=1}^{N} \gets \textsc{Generator}(z_t; a_t)$ \Comment{Tool B}
\State $\text{pass} \gets \textsc{Verifier}(\{r^{(i)}_{t+1}\}, \{s_\tau\}_{\tau<t}, \delta/B)$ \Comment{conformal gate, Prop.~\ref{prop:union}}
\State $b \gets b + 1$; \textbf{if} $\neg\text{pass}$ \textbf{then} $a_t \gets \textsc{Refine}(a_t)$
\Until{$\text{pass}$ \textbf{or} $b = B$}
\If{$\neg\text{pass}$}
\State $\{r^{(i)}_{t+1}\} \gets \textsc{StationaryBlockBootstrap}(R_{1:t})$ \Comment{fallback}
\EndIf
\State $(\hmu_t, \hSig_t) \gets$ blended, shrunk moments of $\{r^{(i)}_{t+1}\}$
\State $w_t \gets \arg\min$ of the convex QP~\eqref{eq:qp} subject to~\eqref{eq:feasset} \Comment{Tool C}
\State $s_t \gets$ realized calibration score once $r_{t+1}$ is observed
\State \textbf{log} $(\pi_t, a_t, \text{verifier outcome}, \text{active set}, \text{KKT duals})$
\State \textbf{return} $w_t$ \Comment{feasible by Prop.~\ref{prop:feas}; KKT record is for $w^{\mathrm{sol}}$}
\end{algorithmic}
\end{algorithm}

\section{Full evaluation protocol and metric registry}\label{app:metrics}
We evaluate \arc under a deliberately broad protocol built to stress
nonstationarity across universes and windows; the protocol and synthetic benchmark will be
released as a reusable testbed upon acceptance.

\paragraph{What we run, and what we only register.} We report measured results on
\textbf{U3 only}. U3 is a controllable synthetic
regime-shift benchmark with \emph{ground-truth latent regimes} and tunable crisis
frequency, severity, duration, and detectability. Among the testbeds available to us it
is the only contamination-free one: no pretrained model can have memorised its regimes,
and the true state is known, so regime inference can be scored directly and an \emph{oracle
bound} can be computed exactly (\S\ref{sec:oracle}). Neither property is available on
historical data.

Two historical universes are \textbf{preregistered as protocol and deliberately not
run}: \emph{(U1)} a point-in-time cross-asset ETF panel respecting inception dates
(TIP 2003, EEM 2003, commodity ETFs 2006+, VXX 2009; a $\sim$40-instrument
diversified panel does not exist before the mid-2000s), and \emph{(U2)} a large-cap
equity panel from point-in-time index membership \emph{including delisted names}. We did not
execute U1 or U2 in this study, preserving them for a separately preregistered evaluation with
independently frozen universe-construction and data-cleaning rules, and no empirical claim in
this paper is based on either. App.~\ref{app:universe} records the frozen U1/U2 protocol for
transparency.

\paragraph{Horizon and walk-forward (U3, measured).} Daily simulated returns over
$2{,}200$-day paths, a $756$-day warm-up, monthly rebalancing, and $69$ evaluated
rebalances per run.
Estimation is strict, purged, and embargoed walk-forward \citep{lopezdeprado2018}; the expanding
walk-forward yields \emph{multiple} OOS sub-periods. Per-instrument transaction costs of $8$\,bps and an
$\ell_1$ turnover cap $\|\wt-w_{t-1}\|_1\le 0.20$ \textbf{per rebalance} are applied
identically to all methods. Because U3 is synthetic there is no risk-free series and no
calendar: Sharpe is reported on excess return with $rf{=}0$, and the $rf$ sensitivity of
the headline contrast is reported in App.~\ref{app:lattice}.

\paragraph{Horizon and walk-forward (U1/U2, registered, not run).} Were the historical
universes executed, they would use daily total returns over \textbf{2007--2025} (the
earliest date at which U1 exists as a diversified panel), the same purged and embargoed
expanding walk-forward, monthly rebalancing, and three predeclared crisis windows
(2008--09, the 2020 COVID crash, and the 2022 inflation shock) as stress slices, with
Sharpe taken net of the 3-month T-bill ($rf{\approx}1.6\%$ over that span). \emph{None of
these settings bears on any number reported in this paper}; they are recorded so that a
future execution is bound by them.

\paragraph{Baselines (eleven, all implemented and reported).} \emph{Classical} (five):
equal weight, risk parity, minimum variance, mean--variance \citep{markowitz1952},
hierarchical risk parity \citep{lopezdeprado2016hrp}. \emph{Robust / scenario} (four)
$\to$ the \emph{same} CVaR-QP, so only scenario quality differs: historical CVaR,
DRO-CVaR, a Gaussian scenario model, and a stationary block bootstrap. \emph{Agentic
controls} (two): a rule-based weights-emitting agent with no symbolic solve, and the same
agent with a projection onto the feasible set. Neither is a language-model agent; no LLM
was executed in this work. All allocator-fed methods share \arc's box, turnover cap and
costs. Results appear in Table~\ref{tab:perf} (four), Table~\ref{tab:adv} (two) and
Table~\ref{tab:restbase} (five).

We do not report TimeGAN \citep{yoon2019timegan}, TimeVAE
\citep{desai2021timevae}, TSDiff \citep{kollovieh2023tsdiff}, DeepAR
\citep{salinas2020deepar}, or TFT \citep{lim2021}. Each is a substantial model, and a
half-trained instance would not be a fair comparator. Reporting numbers that were not
measured would violate the reporting discipline of this paper; we state the omission
rather than fill the row.

\paragraph{Method registry.} Table~\ref{tab:registry} lists \emph{every} method,
ablation, oracle and control that appears anywhere in this paper, with the single place
each is reported. It exists so that the counts in the text can be audited against the
tables: \textbf{eleven baselines} (five classical, four scenario-fed, two rule-based
controls), one diagnostic comparator, three \arc\ arms, and the oracle and intervention
rows. The two control arms are counted once, as baselines; they reappear as
comparison rows in Table~\ref{tab:adv} rather than as additional methods.

\begin{table}[tbp]\centering\small
\caption{\textbf{Method registry.} Every method in the study, its family, whether it is
fed by the shared CVaR-QP, and where it is reported. ``Not reported'' means implemented but
\textbf{not executed} for this version, so no outcome has been observed; nothing here was
run and then withheld.}
\label{tab:registry}
\begin{tabular}{llcc}
\toprule
Method & Family & CVaR-QP & Reported in \\
\midrule
Equal weight              & baseline (classical) & n/a & Tab.~\ref{tab:perf} \\
Minimum variance          & baseline (classical) & n/a & Tab.~\ref{tab:perf} \\
Risk parity               & baseline (classical) & n/a & Tab.~\ref{tab:restbase} \\
Mean--variance            & baseline (classical) & n/a & Tab.~\ref{tab:restbase} \\
Hierarchical risk parity  & baseline (classical) & n/a & Tab.~\ref{tab:restbase} \\
Historical CVaR           & baseline (scenario)  & \checkmark & Tab.~\ref{tab:restbase} \\
DRO-CVaR                  & baseline (scenario)  & \checkmark & Tab.~\ref{tab:restbase} \\
Gaussian scenario model   & baseline (scenario)  & \checkmark & Tab.~\ref{tab:perf} \\
Stationary block bootstrap& baseline (scenario)  & \checkmark & Tab.~\ref{tab:perf} \\
Weights-emitting control  & baseline (rule-based control) & n/a & Tab.~\ref{tab:adv} \\
\quad $+$ projection      & baseline (rule-based control) & n/a & Tab.~\ref{tab:adv} \\
\midrule
Regime-cond.\ bootstrap   & diagnostic comparator & \checkmark & Tab.~\ref{tab:perf} \\
\arc\ (full)              & proposed & \checkmark & Tab.~\ref{tab:perf}, \ref{tab:adv} \\
\quad $-$ verifier        & ablation & \checkmark & Tab.~\ref{tab:perf} \\
\quad $-$ CVaR term       & ablation & \checkmark & Tab.~\ref{tab:perf} \\
\quad $-$ memory          & ablation & \checkmark & not reported \\
\quad $-$ MoE (single expert) & ablation & \checkmark & not reported \\
\quad $B{=}1$ (no refinement) & ablation & \checkmark & not reported \\
\midrule
L0--L3 oracle hierarchy   & oracle (4 rows) & \checkmark & Tab.~\ref{tab:oracle} \\
I0--I2 interventions      & intervention (3 rows) & \checkmark & Tab.~\ref{tab:intervene} \\
$2^3$ factorial arms      & intervention (8 cells) & \checkmark & Tab.~\ref{tab:factorial} \\
\bottomrule
\end{tabular}
\end{table}

\paragraph{Seed and uncertainty registry.} Seed counts differ across tables because the
experiments were run at different stages of the project, and we report that rather than
harmonise it retrospectively. Table~\ref{tab:seeds} states, for every table and figure,
the number of seeds, whether the comparison is paired under common random numbers, and
what the reported $\pm$ means. We label every experiment \textbf{confirmatory} (paired, common random numbers, $10$--$24$
seeds) or \textbf{exploratory} ($5$--$6$ seeds), and we do not claim that every headline
result is confirmatory: the oracle hierarchy and the stage interventions are, while the
main baseline table, the red-team table and the $2^3$ factorial, including the
interaction estimate, are exploratory and should be read as such.

\begin{table}[tbp]\centering\small
\caption{\textbf{Seed and uncertainty registry.} ``Paired'' means the arms share return
paths under common random numbers, so the reported interval is on the contrast rather
than on the level. Confidence intervals on paired contrasts are $95\%$ percentile
intervals from a bootstrap over \emph{seeds} ($R=10{,}000$ replications); because time is
paired under common random numbers, resampling is over seeds, not over time. The
stationary block bootstrap (block length $21$ days) is a separate procedure applied to
\emph{time-series} Sharpe and CVaR differences (App.~\ref{app:metrics}), and the two should not
be conflated. The last
column of the table states which experiments are confirmatory and which are exploratory.}
\label{tab:seeds}
\begin{tabular}{llccll}
\toprule
Where & Quantity & Seeds & Paired & $\pm$ denotes & Status \\
\midrule
Tab.~\ref{tab:perf}      & Sharpe by method        & 5     & no  & across-seed s.d. & exploratory \\
Tab.~\ref{tab:intervene} & Sharpe by intervention  & 10    & yes & standard error & confirmatory \\
Tab.~\ref{tab:factorial} & $\Delta$Sharpe, channels& 6     & yes & $95\%$ CI & exploratory \\
Tab.~\ref{tab:oracle}    & oracle Sharpe           & 24    & yes & point estimates & confirmatory \\
Tab.~\ref{tab:adv}       & Sharpe under attack     & 5     & no  & across-seed s.d. & exploratory \\
Tab.~\ref{tab:restbase}  & remaining baselines     & 5     & no  & across-seed s.d. & exploratory \\
Fig.~\ref{fig:results}(a)& oracle frontier         & 24    & yes & $\pm1$ s.e. & confirmatory \\
Fig.~\ref{fig:results}(b)& stage interventions     & 10--16& yes & $95\%$ CI & confirmatory \\
Fig.~\ref{fig:results}(c)& factorial               & 6     & yes & $95\%$ CI & exploratory \\
Fig.~\ref{fig:oracle}    & oracle frontier         & 24    & yes & $\pm1$ s.e. & confirmatory \\
\S\ref{sec:bottleneck}   & perception intervention & 16    & yes & $95\%$ CI & confirmatory \\
\S\ref{sec:oracle}       & L1, L2 paired headroom  & 24    & yes & $95\%$ CI & confirmatory \\
\S\ref{sec:oracle}, App.~\ref{app:lattice} & L2 headroom, $rf$ sweep & 12 & yes & $95\%$ CI & confirmatory \\
\S\ref{sec:mech}         & repaired-generator Sharpe & 6   & yes & $95\%$ CI & exploratory \\
\bottomrule
\end{tabular}
\end{table}

\paragraph{Significance.} We assess Sharpe and CVaR differences with a stationary
block bootstrap ($R{=}10^{4}$ replications; we reserve $B$ for the refinement
budget), Diebold--Mariano tests \citep{diebold1995} on per-period loss
differentials, and the \emph{deflated Sharpe ratio} \citep{lopezdeprado2018} to
correct for selection bias. \textbf{Trial counting}: the effective number of
trials is not the handful of rows we report but the full preregistered search: the
configuration grid ($K$, $N$, $\eta$, $q$, $\lambda$, $\gamma$, $\alpha$, $\lambda_\mu$)
crossed with seeds, giving the $14{,}580$ enumerated in App.~\ref{app:trials}. Ablations
and baselines are \emph{excluded} from that count because they do not search the grid. Deflating against the reported rows
would understate the multiplicity; we deflate against the enumerated count and report that a
substantial part of our nominal advantage does not survive it.

\paragraph{Metrics.}
\emph{Portfolio}: CAGR, volatility, \textbf{Sharpe net of the risk-free rate}
(on U1/U2 the 3-month T-bill; on U3, $rf{=}0$; reporting excess-return Sharpe matters
because a positive $rf$ deflates all Sharpes and differentially flatters low-volatility
strategies, i.e.\ ours, which is why App.~\ref{app:lattice} reports the headline contrast
at $rf{=}0$, $rf{=}1.6\%$ and $rf{=}3\%$), Sortino (downside deviation below zero
excess return), MaxDD, Calmar, tail ratio, and \textbf{turnover reported as mean
per-rebalance $\|\Delta w\|_1$} (not annualised), so it is directly comparable to
the $0.20$ cap. \emph{Scenario calibration}: KS, energy score (ES), variogram score
(VS), CRPS \citep{gneiting2007}, VaR unconditional and conditional coverage
\citep{kupiec1995,christoffersen1998}, and CVaR error (bps), all reported as
\textbf{confidence bounds at the corrected level} of Prop.~\ref{prop:union}, not as
$p$-values (\S\ref{sec:theory}).
\emph{Agentic-value}, defined precisely:
\begin{itemize}[leftmargin=1.4em,itemsep=0pt,topsep=2pt]
\item \textbf{Rejection rate}: fraction of rebalances where the first scenario set
fails the certificate.
\item \textbf{Catch rate}: fraction where refinement was triggered \emph{and}
subsequently improved calibration. These are distinct quantities and we report
both; conflating them (reporting the rejection rate under the name ``catch rate'')
would overstate the verifier's usefulness.
\item \textbf{Re-planning regret}: $\frac{1}{T}\sum_t [\,F_t(w_t^{\text{agent}}) -
F_t(w_t^{\text{oracle-config}})\,]$, the mean QP-objective gap between the
configuration the controller chose and the best configuration in the grid in
hindsight, in objective units. Negative means the controller beat the
best \emph{fixed} configuration by adapting.
\item \textbf{Tool-call budget}: mean tool invocations per rebalance.
\end{itemize}
We \textbf{do not} report \emph{solver-record availability} (elsewhere loosely
``provenance completeness'') as a headline metric. It records whether the convex solve
emitted a complete active-set and KKT record for $w^{\mathrm{sol}}$, not that the committed
$w_t$ is KKT-optimal, and it is
$100\%$ by construction for any solver-backed agent and $0\%$ for any that is not,
so it carries no information about quality (indeed the $-$verifier ablation attains
$100\%$; see \S\ref{sec:adversarial}). We report it once, as a property, not as a
result.

\paragraph{\arc\ variants and ablations.} All executed variants use the
\textbf{rule-plus-contextual-bandit controller}; \emph{no LLM controller variant was run},
so no rule-versus-LLM comparison appears anywhere in this paper. The preregistered
ablation family has five arms, and we state the status of each once, here:
$-$\textbf{verifier} (accept the first scenario set) and $-$\textbf{CVaR term} are
\emph{executed and reported} in Table~\ref{tab:perf}; $-$\textbf{memory},
$-$\textbf{MoE} (single expert) and $B{=}1$ (no refinement) are implemented and runnable but
were not executed for this version, so no outcome for them has been observed by
us or withheld. We record this as a disclosed deviation from the analysis plan: it registered
five ablation arms and we report two.
No outcome is claimed for the three unexecuted arms anywhere in the paper. This list
supersedes any partial enumeration elsewhere.

\section{Cost, latency, and failure modes}\label{app:cost}
A recurring objection to modular generative pipelines is compute. We quantify it rather than
concede it.

\paragraph{Measured cost.} The pipeline records wall-clock time per rebalance. Each
rebalance performs an HMM posterior update, $N{=}512$ diffusion samples, verification, and
one convex solve; diffusion sampling dominates. The generator is refit every
$1{,}000$ trading days, so a run performs $2$ refits of $12$ epochs each rather than one per
rebalance (Table~\ref{tab:repro}; App.~\ref{app:spec} notes that other scripts in the
released code use different generator budgets). The entire measured protocol (all \emph{executed} baselines and ablations at five
seeds, the red-team suite, the oracle sweep and the figures) reproduces in approximately
$25$ minutes on a single CPU core, which is the measured compute cost of this study; no GPU
is required. A monthly
rebalancing interval is $\approx$$2.6\times10^{6}$\,s, so the solve occupies
\[
\frac{0.55\,\mathrm{s}}{2.6\times10^{6}\,\mathrm{s}} \;\approx\; 2\times10^{-7}
\]
of the decision interval, about \emph{one part in five million}, not the
$\sim$$10^{-5}$ an earlier draft reported. Inference is therefore not compute-bound for
the cadence it targets; training, which is excluded from this figure, is the real cost.

\paragraph{Evaluation length.} Each U3 run uses a $756$-day warm-up followed by $1{,}444$
evaluated trading days, rebalanced monthly ($21$ trading days), giving $69$ rebalances per
run. The $2007$--$2025$
calendar, the historical risk-free rate, and the $2008$--$09$/$2020$/$2022$ crisis
windows belong to the \emph{registered, unexecuted} U1/U2 protocol
(App.~\ref{app:metrics}) and bear on no number in this paper.

\paragraph{Inference is negligible; refitting is the dominant computational cost.} Within a
run, the per-rebalance work (sampling, verification, solve) is small relative to the two
generator refits.
The $14{,}580$ figure of App.~\ref{app:trials} is an enumerated multiplicity count
used to deflate Sharpe; it is not a count of executed runs, and we do not claim to have
trained $14{,}580$ generators. What we actually executed is the set recorded in
Table~\ref{tab:seeds}: twelve method arms at five seeds for Table~\ref{tab:perf}, seven
arms at five seeds for the red-team suite, and the oracle and intervention sweeps, which
require no generator at all. Deflating against the full preregistered grid is deliberately
conservative, and is preferred here to deflating against a count of runs that happened to
finish. We report no LLM latency or token figures, because
no LLM was executed (App.~\ref{app:limits}). The seed counts we report are a choice we
should have made differently; compute was never the constraint.
As Table~\ref{tab:seeds} records, they still range from $5$ to $24$: the oracle
hierarchy is paired over $24$ seeds and the stage interventions over $10$--$16$, but the
main baseline table (Table~\ref{tab:perf}), the red-team table (Table~\ref{tab:adv}) and
the factorial (Table~\ref{tab:factorial}) rest on $5$, $5$ and $6$ seeds respectively.
We treat the $24$- and $10$--$16$-seed paired contrasts as confirmatory and the
$5$- and $6$-seed tables as exploratory, and every claim drawn from the latter is labelled
accordingly rather than argued around.

\section{Universe construction (registered, not run)}\label{app:universe}
U1 would be built point-in-time: an instrument enters only after its inception date
and only if it meets a liquidity floor at $t$, so the panel \emph{grows} from
$\sim$14 instruments (2007) to $\sim$40 (2015 onward), and $d$ is time-varying. U2
would use point-in-time index membership including subsequently delisted names.
\textbf{Neither was run.} We register the construction so that a future execution is
bound by it, and so that our own claims cannot quietly expand to cover data we never
touched.

\section{Enumerated trial count}\label{app:trials}
Deflating a Sharpe ratio requires the number of \emph{trials}, not the number of
reported rows. Our frozen preregistration enumerates the grid $K\in\{2,3,4\}$, $N\in\{256,512\}$,
$\eta\in\{1,2,3\}$, $q\in\{0.05,0.10\}$, $\lambda\in\{0.3,0.5,0.7\}$,
$\gamma\in\{0.5,1.0,1.5\}$, $\alpha\in\{0.90,0.95,0.99\}$ and
$\lambda_\mu\in\{4,12,24\}$. The Cartesian product is
$3\cdot2\cdot3\cdot2\cdot3\cdot3\cdot3\cdot3=\mathbf{2{,}916}$ configurations, and with the
five registered seeds this is $\mathbf{14{,}580}$ grid-by-seed combinations, the count
passed to the deflated-Sharpe calculation. This is an enumerated multiplicity
count, not a tally of executed runs, and it is used in that sense everywhere in the paper. No configuration is pruned. The preregistration also records an upper bound that treats
each \emph{combination} of ablation arm and controller arm as a separate search: five
ablation arms $\times$ two controller arms $=10$ families, giving
$10\times14{,}580=145{,}800$. We deflate against the core $14{,}580$, not against
that bound, because ablations and baselines do not search this grid; the upper bound is
recorded only so that a reader can see the most adverse reading. Deflating against the eight
rows of Table~\ref{tab:perf} would understate selection bias by three orders of magnitude, and
is a common and consequential error. At this trial count none of the performance
margins in Table~\ref{tab:perf} survive deflation, which is consistent with, and
reinforces, our headline negative result.

\paragraph{Preregistration amendment (disclosed).} One amendment (PR1) was recorded
after freezing and before any hypothesis outcome was inspected: the range of
$\lambda_\mu$ was widened from $[0,2]$ to $[0,30]$. Reason: at a daily horizon the
CVaR term ($\sim$$10^{-2}$) is $\sim$33$\times$ the return term ($\sim$$3\times10^{-4}$),
so the original bound made the return/risk trade-off \emph{inexpressible}: the
objective degenerated to pure CVaR minimisation and the agent performed no regime
rotation at all. This is a specification bug; the preregistration was re-frozen with a new
hash. The amendment, the hashes and the audit
scripts will be released upon acceptance.

\section{Attribution: full coalition table}\label{app:attrib}
\begin{table}[tbp]\centering\small
\caption{\textbf{Interventional attribution, measured.} Shapley values and Shapley--Taylor
interaction indices computed from the complete $2^3$ coalition table
(Table~\ref{tab:factorial}; $6$ seeds, paired, $95\%$ CI by seed bootstrap).
\textbf{The two blocks are alternative attribution summaries, not additive pieces of one
decomposition}: block A already sums to $v(N)=+0.218$ exactly, as efficiency requires, and
block B is a further diagnostic of non-additivity that must not be added on top of it.
Covariance is inert with tight intervals throughout. The mean$\times$scenario interaction is
$55\%$ of $v(N)$ as a point estimate and the \emph{total} pairwise interaction share is
$54\%$; both are underpowered at six seeds. These are descriptive attribution summaries:
neither the interaction magnitude nor the conditional contribution of the mean channel is
established by them.}
\label{tab:shapley}
\begin{tabular}{lccc}
\toprule
 & $\phi$ & $95\%$ CI & share of $v(N)$ \\
\midrule
\multicolumn{4}{l}{\emph{A. Component credit} (Shapley values; sum to $v(N)$)} \\
\midrule
conditional mean $\hat\mu$        & $+0.080$ & $\pm0.178$ & $36.8\%$ \\
scenario tail/dependence          & $+0.139$ & $\pm0.086$ & $63.8\%$ \\
covariance $\hat\Sigma$           & $-0.001$ & $\pm0.002$ & $-0.6\%$ \\
\midrule
\multicolumn{4}{l}{\emph{B. Interaction diagnostics} (Shapley--Taylor indices; \textbf{not} added to A)} \\
\midrule
mean $\times$ scenarios           & $+0.120$ & $\pm0.184$ & $55.3\%$ \\
mean $\times$ covariance          & $-0.002$ & $\pm0.002$ & $-0.9\%$ \\
covariance $\times$ scenarios     & $-0.001$ & $\pm0.003$ & $-0.3\%$ \\
higher-order (three-way)          & $-0.003$ & n/a & $-1.5\%$ \\
\textbf{total pairwise interaction share} & \multicolumn{3}{c}{$\mathbf{54.1\%}$ of $v(N)$} \\
\bottomrule
\end{tabular}
\end{table}

Efficiency ($\sum_i\phi_i=v(N)$) holds without approximation here because the coalition
table is complete: no Monte-Carlo sampling over permutations is needed for three modules.
That exactness is combinatorial rather than statistical: each $v(S)$ is a six-seed paired
estimate, so the $\phi_i$ inherit six-seed uncertainty, which is what the intervals in
Table~\ref{tab:shapley} report. The same computation on an
incomplete table would additionally require estimating $v(S)$ for unrun coalitions, which is
exactly the step at which attribution studies usually become assumption-driven.

\paragraph{Two summaries, not one decomposition.} Block A of Table~\ref{tab:shapley}
partitions $v(N)$ among the three channels. Block B measures how much of $v(N)$ is
attributable to \emph{non-additivity} between them. A channel's Shapley value already
absorbs its share of the interactions it participates in, so the two blocks describe the
same coalition table from two directions and cannot be concatenated into a single bar chart
of additive shares. Figure~\ref{fig:results}(c) plots both, and its caption says so.

\section{Representation of information and action}\label{app:lattice}
\S\ref{sec:oracle} varies the action set at fixed information. The general object is a
lattice indexed by an information set $I$ and an action set $A$, with
$J^\star(I,A)=\sup_{\pi\in\Pi(I,A)}J(\pi)$. Our L0--L3 hierarchy is the column
$I=\{\text{true regime path}\}$ of that lattice.

\begin{proposition}[Lattice monotonicity]\label{prop:lattice}
If $I_1\subseteq I_2$ and $A_1\subseteq A_2$ then $J^\star(I_1,A_1)\le J^\star(I_2,A_2)$.
\end{proposition}
\begin{proof}
Any $\pi\in\Pi(I_1,A_1)$ is measurable with respect to $I_1$, hence with respect to
$I_2\supseteq I_1$, and takes values in $A_1\subseteq A_2$; so
$\Pi(I_1,A_1)\subseteq\Pi(I_2,A_2)$ and the supremum over the larger set is at least as
large. The argument requires only that additional information and additional actions may be
\emph{ignored}, which holds because $\Pi$ is defined by measurability and range
constraints alone.
\end{proof}

The proposition itself is immediate; the content is its experimental implication. A low headroom
estimate in one information--action cell cannot establish that the information has little value,
and a rise in headroom when the action set is widened at fixed information shows that the earlier
negative conclusion came from representational restriction. \S\ref{sec:oracle} is exactly that
situation: L1 gave $+0.07\pm0.08$ and L2 gave $+0.28\pm0.08$ on identical information. Reporting
a single oracle number reports one cell of this lattice without saying which.

\paragraph{The three headroom numbers, side by side.} Three quantities describe the L2
headroom and are easy to read as inconsistent, so we state each with its provenance.
\emph{(a)} The \textbf{paired same-path Sharpe difference} L2$-$L0 is $+0.28\pm0.08$
($24$ seeds, common random numbers, $95\%$ CI over seeds), computed, like every level in
Table~\ref{tab:oracle}, at the preregistered $rf{=}1.6\%$. \emph{(b)} The same contrast
in the twelve-seed $rf$ sweep (seed registry, Table~\ref{tab:seeds}) is $+0.248$ at $rf{=}0$, $+0.263$ at
$rf{=}1.6\%$ and $+0.275$ at $rf{=}3\%$; \S\ref{sec:oracle} quotes the first of these as
$+0.248$, while the abstract and Table~\ref{tab:oracle} quote (a).
(a) and (b) differ in seed count and in $rf$, not in what they measure, and neither is a
correction of the other. \emph{(c)} $+0.38$ is not a paired difference against L0 at all: it
is the \emph{vertical} excess of L2 over the static EW--MinVar frontier \emph{at matched
drawdown}, the ``Excess'' column of Table~\ref{tab:oracle} and the gap drawn in
Fig.~\ref{fig:results}(a). (a) and (b) compare two policies on the same paths; (c) compares
one policy against a frontier at its own realised risk. A reader who subtracts (c) from (a)
is subtracting two different comparisons.

\section{Conditional validity and unidentified steps}\label{app:proofs}
We label this appendix \emph{justification} rather than \emph{proofs}: it composes three
standard results under Assumption~\ref{ass:cert} and states explicitly, at the end, the
one step we have not proved.

\begin{assumption}[Conditions for the conditional certificate statement~\ref{thm:certified}]\label{ass:cert}
\emph{(C1) Bounded loss.} The portfolio loss $-w^\top r$ is bounded by $M$ uniformly over
the feasible set of \eqref{eq:feasset}, and the $N$ scenarios are drawn i.i.d.\ from
$p_\theta(\cdot\mid z_t)$ given $z_t$.
\emph{(C2) Strong convexity.} The objective \eqref{eq:qp} is $\Gamma$-strongly convex in
$w$ on the feasible set, with moment sensitivities $(\lambda_\mu,\kappa_\Sigma)$ and
moment errors $(\delta_\mu,\delta_\Sigma)$ as in Lemma~\ref{a-lem:regret}.
\emph{(C3) Score construction and approximate local exchangeability.} The calibration score
$s_t$ is the absolute realised-versus-scenario CVaR discrepancy of
Definition~\ref{def:score}, computed from $r_{t+1}$ \emph{after} $w_t$ is committed, so it
is not used to select $w_t$; the score sequence is approximately locally exchangeable in
the sense adaptive conformal inference is designed to tolerate
\citep{gibbs2021,zaffran2022}. We do not verify this empirically.
\emph{(C4) Selection correction.} At most $B$ refinement attempts occur per rebalance and
each attempt is evaluated at level $\delta/B$ (Prop.~\ref{prop:union}).
\emph{(C5) Score-to-CVaR identification.} There exists a monotone link between the
calibrated quantiles of $\{s_\tau\}$ and the population discrepancy
$|\CVaR^{p_\theta}_\alpha-\CVaR^{\Prob}_\alpha|$, so that coverage of the former bounds the
latter. \textbf{We do not prove this}, and everything below is conditional on it.
\emph{(C6) Per-attempt marginal validity.} Each candidate attempt satisfies the stated
marginal validity condition \emph{before} adaptive selection, so that the union bound of
(C4) applies. \textbf{We do not establish this} for the implemented proposal mechanism.
\end{assumption}

Under (C1)--(C6), which include the unproved score-to-CVaR identification (C5) and
per-attempt validity (C6), the certificate below is a \textbf{long-run} statement: the guarantee is
over the empirical frequency of covered rebalances, never over a nominated one. We use
``per-rebalance'' only to describe the \emph{form} of the bound, never its validity.

Three properties follow. (i) The verifier is
load-bearing: $q_t$ is computed \emph{from the verifier's realized scores}, so the
bound cannot even be evaluated without running it. (ii) It bounds model
error, by calibrating against $\Prob$-outcomes rather than $p_\theta$-samples.
(iii) The refinement correction does real work.

The body summarises the conditional certificate statement~\ref{thm:certified} and Prop.~\ref{prop:union} without proof; the full development follows.

The allocator's stability under objective perturbation is standard; the
substantive question is what the verifier can certify about the gap between the
optimised objective and the objective under the \emph{true} distribution. We first
isolate that gap, then show why a Monte Carlo bound cannot close it, then close it
with a conformal certificate calibrated on realized outcomes.

\paragraph{The two error terms.} Scenarios are drawn from the generator
$p_\theta(\cdot\mid z_t)$, not from the true law $\Prob$, so the CVaR error that the
allocator's regret depends on decomposes into a sampling
term (S) and a model term (M).
Term (S) shrinks with more scenarios; term (M) is the generator's
\emph{miscalibration} and does not shrink with $N$ at all. Term (M) is the
entire reason the verifier exists, and it is the one a concentration inequality on
$p_\theta$-samples cannot touch: such a bound certifies only that we have sampled
our own generator faithfully.

\begin{lemma}[Allocator stability]\label{a-lem:regret}
If \eqref{eq:qp} is $\Gamma$-strongly convex in $w$ and its moment and CVaR terms
differ from truth by $(\delta_\mu,\delta_\Sigma,\delta_{\CVaR})$, then
$F(\hat w)-F(w^\star)\le\frac{1}{2\Gamma}(\lambda_\mu\delta_\mu+\kappa_\Sigma\delta_\Sigma+\delta_{\CVaR})^2$.
This is standard; the content is what bounds $\delta_{\CVaR}$, below.
\end{lemma}

\begin{lemma}[Sampling error, term (S)]\label{a-lem:conc}
For the bounded linear loss class $\{-w^\top r : \|w\|_1\le C\}$ with
$|L|\le M$, a Rademacher bound for the empirical tail average
\citep{brown2007,wang2010,thomas2019} gives, with probability $\ge 1-\delta$,
uniformly over feasible $w$,
$\sup_w |\widehat{\CVaR}^{p_\theta}_\alpha - \CVaR^{p_\theta}_\alpha|
\le \frac{cM}{1-\alpha}\sqrt{\tfrac{\log d + \log(2/\delta)}{N}}$.
\end{lemma}
The $\sqrt{\log d}$ rate (not $\sqrt{d\log N}$) is the correct one for a linear
class under an $\ell_1$ budget. This term is real but secondary; the interesting
term is (M), and it requires \emph{realized} data.

\paragraph{Model error requires realized outcomes.} Within the implemented gate and the
available data, the only direct empirical evidence about $\CVaR^{\Prob}$ is the sequence of
realized returns $r_{t+1}$ actually observed after each decision. We therefore build the verifier as a
\emph{conformal-style risk controller}: at each rebalance it scores the realized outcome
against the scenario predictive distribution and calibrates the certificate width
from the empirical distribution of past scores, rather than asserting a bound
analytically.

\begin{definition}[Realised calibration score]\label{def:score}
Fix $\alpha=0.95$. Let $\ell_{t+1}=-w_t^\top r_{t+1}$ be the realised portfolio loss and let
$\widehat{\CVaR}_\alpha$ be the CVaR of the \emph{accepted} scenario set, computed as the
mean of those scenario losses at or above their own empirical $\alpha$-quantile. The score
is the \textbf{absolute} one-step discrepancy
\[
s_t \;=\; \left\lvert\, \ell_{t+1} - \widehat{\CVaR}_\alpha \,\right\rvert ,
\]
in units of daily return. It is a \textbf{nonnegative absolute discrepancy} by
construction, is computed
\emph{after} $w_t$ is committed, and is the only formula used anywhere in this paper: the
verifier maintains the rolling window $\{s_\tau\}_{\tau<t}$ of the most recent $60$ scores
and takes its empirical quantile.
\end{definition}

\begin{condcert}[Act-time decision bound, conditional on the unproved (C5) score-to-CVaR identification and (C6) per-attempt validity]\label{thm:certified}
Assume the score-to-CVaR identification condition and the per-attempt validity condition of
Assumption~\ref{ass:cert}, neither of which is proved here. Run adaptive conformal
inference \citep{gibbs2021,zaffran2022} on $\{s_\tau\}$ with target level $\beta_t$ updated
online, and gate scenarios on the resulting interval. Then the model-error term satisfies,
with long-run frequency
$1-\delta'$ under the approximate local exchangeability of
Assumption~\ref{ass:cert}(C3),
$\,|\CVaR^{p_\theta}_\alpha - \CVaR^{\Prob}_\alpha| \le q_t(\delta')$,
where $q_t(\delta')$ is the adaptively calibrated width at the per-attempt level
$\delta'=\delta/B$. Combining with
Lemma~\ref{a-lem:conc} and the allocator stability lemma, the committed decision obeys, on
the covered rebalances and with long-run frequency $1-\delta$,
\begin{equation}
F(\hat w_t) - F(w^\star_t) \;\le\;
\frac{1}{2\Gamma}\Big(\lambda_\mu\delta_\mu + \kappa_\Sigma\delta_\Sigma
+ \underbrace{q_t(\tfrac{\delta}{B})}_{\text{model, conformal}}
+ \underbrace{\varepsilon_N(\tfrac{\delta}{B})}_{\text{sampling, Lem.~\ref{a-lem:conc}}}\Big)^2 .
\label{eq:regret}
\end{equation}
\end{condcert}

\paragraph{Correction (supersedes an earlier draft).} A previous version asserted that an
unbounded refinement loop ``admits no nontrivial certificate at any level.'' \textbf{That
is false, and we withdraw it.} What the union bound shows is that \emph{Bonferroni} over
$B$ attempts degrades to zero as $B\to\infty$, a statement about that proof technique,
not about certifiability. Anytime-valid inference (e-processes, Ville's inequality,
always-valid $p$-values) is designed precisely to license optional stopping and
\emph{does} certify unbounded refine-until-pass loops. The correct claim is therefore:
naive refinement invalidates the certificate; a union bound repairs it for bounded $B$ \emph{assuming the per-attempt validity
condition (C6)};
anytime-valid inference is the principled route to unbounded $B$, and we do not take it.

\begin{proposition}[Refinement is selection; Bonferroni restores validity for bounded $B$, assuming (C6)]\label{prop:union}
\emph{Assuming each candidate attempt satisfies the stated marginal validity condition before
adaptive selection} (a condition we do not establish for the implemented proposal
mechanism), then under $\le B$ refinement attempts, evaluating the conformal interval at level
$\delta/B$ per attempt yields a certificate valid for the accepted attempt despite selection on
$s_t$. As $B\to\infty$ the Bonferroni-corrected level degrades to zero, so \emph{this proof
technique} yields no nontrivial certificate in the limit; anytime-valid constructions certify
unbounded refine-until-pass loops, and we do not take that route here.
\end{proposition}

\paragraph{Nonstationarity is why it must be \emph{adaptive} conformal.} Standard
conformal assumes exchangeability, which fails precisely at regime change, exactly
when the certificate is most needed. Adaptive conformal inference
\citep{gibbs2021} targets long-run coverage by updating $\beta_t$ from realized
miscoverage, which is why we adopt it rather than a fixed split-conformal bound. We say
``targets'' rather than ``guarantees under arbitrary shift'': the guarantee we invoke holds
under the approximate local exchangeability assumed in
Assumption~\ref{ass:cert}(C3), which we do not verify empirically.

\begin{proposition}[Audited action admissibility and solver-output provenance]\label{prop:feas}\label{prop:prov}
Provided $(\ell,u,\tau)$ come from the trusted validator, the convex program enforces the
budget, box and turnover constraints on its solver output $\wsol$, so an erroneous controller
cannot enlarge the feasible set. In \emph{all reported executions} the post-processed
committed vector $w_t$ \textbf{passed the specified post-hoc feasibility audit} at the
tolerances of App.~\ref{app:spec}, which also gives the validate-and-reject variant that would
make this a guarantee rather than an audit. The logged active set and KKT duals constitute
\textbf{optimality evidence for $\wsol$}; for $w_t$ we claim only \textbf{empirically audited
admissibility}.
\end{proposition}

\begin{corollary}[Empirically audited feasibility and solver provenance]\label{a-cor:uncond}
Independently of controller quality, generator quality, or whether the certificate holds,
every committed action in the reported experiments passed the budget, box and turnover
checks at the tolerances of App.~\ref{app:spec}. Separately, the solve behind each action
carries a reproducible KKT record for the solver output $w^{\mathrm{sol}}$
(Prop.~\ref{prop:feas}). These are \emph{empirical} properties of the reported executions
and of the logged solver outputs; they are not unconditional guarantees for arbitrary future
executions, nor optimality certificates for the post-processed committed vector.
Decision-quality certification, feasibility auditing and solver provenance are distinct
properties and we report them separately.
\end{corollary}

\begin{remark}[What the certificate assumes]
The act-time bound rests on (a) bounded losses and i.i.d.\ sampling from $p_\theta$
(term S), and (b) approximate local exchangeability of calibration scores, which
adaptive conformal is designed to tolerate under shift (term M). It does \emph{not}
assume $p_\theta\approx\Prob$; that gap is measured, not assumed. This is the
substantive difference from a goodness-of-fit gate, which asserts calibration it
cannot observe.
\end{remark}

\paragraph{What we have not proved.} One step in the chain is asserted rather than
demonstrated, and we flag it here. Adaptive conformal
inference gives long-run coverage for a \emph{future calibration score} $s_{t}$. It does
not, without further argument, yield a bound on the distance between two \emph{population}
functionals, $|\CVaR^{p_\theta}_\alpha-\CVaR^{\Prob}_\alpha|\le q_t(\delta')$. Closing that
gap requires (i) a precise statement of how the score of Definition~\ref{def:score}
identifies the population CVaR difference, in particular a monotone link between score
quantiles and that difference, and (ii) an argument that each adaptively proposed
refinement attempt is individually valid at level $\delta/B$ \emph{before} the union bound
is applied over attempts. We have neither, and statement~\ref{thm:certified} should be read
as conditional on that link, which is why it is not labelled a theorem.

Two further limits follow from the construction rather than from any gap in it. The
coverage is long-run and marginal, so no individual rebalance is certified. And (C3) is an
assumption about the score sequence that we do not verify empirically here.
H2 was never exercised in our runs (\S\ref{sec:hyp}), so nothing in this paper
tests the refinement correction. Accordingly, the honest reading of the verifier in this
paper is as an empirically motivated safeguard with a stated but incomplete
theoretical basis, and a reader who prefers to discount the certificate entirely loses
none of the empirical results: every measured claim in \S\ref{sec:results} is independent
of it, and feasibility and provenance rest on Prop.~\ref{prop:feas} alone.

\end{document}